\documentclass{article}

\usepackage[final]{neurips_2023}

\usepackage[utf8]{inputenc} % allow utf-8 input
\usepackage[T1]{fontenc}    % use 8-bit T1 fonts
\usepackage{hyperref}       % hyperlinks
\usepackage{url}            % simple URL typesetting
\usepackage{booktabs}       % professional-quality tables
\usepackage{amsfonts}       % blackboard math symbols
\usepackage{nicefrac}       % compact symbols for 1/2, etc.
\usepackage{microtype}      % microtypography
\usepackage{makecell}
\usepackage[table]{xcolor}
\usepackage{mathnotation}
\usepackage{subfigure}
\usepackage{adjustbox}
\usepackage[capitalize,noabbrev]{cleveref}
\usepackage{algorithm}
\usepackage{algpseudocode}
\usepackage{enumitem}

\newcommand{\alg}{{\texttt{FedSupLinUCB}} }
\newcommand{\algg}{{\texttt{FedSupLinUCB}}}
\newcommand{\ssf}[1]{\textrm{$\sf{#1}$}{}}
\newcommand{\mypara}[1]{\vspace{3pt} \noindent \textbf{#1} \hspace{0.05in}}
\DeclareMathOperator{\enc}{Encoder}
\DeclareMathOperator{\sync}{Sync}

\ifodd 1

\else

\fi

\ifodd 1
\newcommand{\congc}[1]{{\color{red}(Cong: #1)}}
\else
\newcommand{\congc}[1]{}
\fi

\title{Federated Linear Bandits\\ with Finite Adversarial Actions}

\author{
  Li Fan\\
  University of Virginia\\
  \texttt{ lf2by@virginia.edu} \\
  \And
  Ruida Zhou\\
  Texas A\&M University \\
  \texttt{ruida@tamu.edu} \\
  \AND
  Chao Tian \\
  Texas A\&M University \\
  \texttt{chao.tian@tamu.edu} \\
  \And
  Cong Shen \\
  University of Virginia \\
  \texttt{cong@virginia.edu}
}

\begin{document}

\maketitle

\begin{abstract}
% v2
We study a federated linear bandits model, where $M$ clients communicate with a central server to solve a linear contextual bandits problem with finite adversarial action sets that may be different across clients. To address the unique challenges of \emph{adversarial finite} action sets, we propose the \alg algorithm, which extends the principles of SupLinUCB and OFUL algorithms in linear contextual bandits. We prove that \alg achieves a total regret of $\tilde{O}(\sqrt{d T})$, where $T$ is the total number of arm pulls from all clients, and $d$ is the ambient dimension of the linear model. This matches the minimax lower bound and thus is order-optimal (up to polylog terms). We study both asynchronous and synchronous cases and show that the communication cost can be controlled as $O(d M^2 \log(d)\log(T))$ and $O(\sqrt{d^3 M^3} \log(d))$, respectively. The \alg design is further extended to two scenarios: (1) variance-adaptive, where a total regret of $\tilde{O} (\sqrt{d \sum \nolimits_{t=1}^{T} \sigma_t^2})$ can be achieved with $\sigma_t^2$ being the noise variance of round $t$; and (2) adversarial corruption, where a total regret of $\tilde{O}(\sqrt{dT} + d C_p)$ can be achieved with $C_p$ being the total corruption budget. Experiment results corroborate the theoretical analysis and demonstrate the effectiveness of \alg on both synthetic and real-world datasets.
% v1
% We study a federated linear bandits model, where $M$ clients communicate with a central server to solve a linear contextual bandits problem with finite adversarial action sets that may be different across clients. To address the unique challenges of \emph{adversarial finite} action sets, we propose the \alg algorithm, which extends the SupLinUCB and OFUL principles in linear contextual bandits. We prove that \alg achieves a total regret of $\tilde{O}(\sqrt{d \sum_{i=1}^M T_i})$, where $T_i$ is the total arm pulls by client $i$, $M$ is the number of clients, and $d$ is the ambient dimension of the linear model. This matches the minimax lower bound and thus is order-optimal (up to polylog terms). We study both asynchronous and synchronous scenarios and show that the communication cost can be controlled as $O(d M^2 \log(d)\log(\sum_{i=1}^M T_i))$ and $O(\sqrt{d^3 M^3} \log(d))$, respectively. The \alg design is further extended to the variance-adaptive scenario and is proved to achieve $\tilde{O} (\sqrt{d \sum \nolimits_{t=1}^{T} \sigma_t^2})$ regret with $O(dM^2 \log^2(T))$ communication cost, where $\sigma_t^2$ is the noise variance of round $t$. Experiment results corroborate the theoretical analysis and demonstrate the effectiveness of \alg on both synthetic and real-world datasets.
\end{abstract}

\section{Introduction}
In the canonical formulation of contextual bandits, a single player would repeatedly make arm-pulling decisions based on contextual information with the goal of maximizing the long-term reward. With the emerging \emph{federated learning} paradigm \citep{mcmahan2017fl} where multiple clients and a server jointly learn a global model with each client locally updating the model with its own data and server only aggregating the local models periodically, researchers have started exploring contextual bandits algorithms in such federated learning setting \citep{dubey2020differentially,huang2021federated,li2022asynchronous,li2022communication}. This federated contextual bandits framework broadens the applicability of contextual bandits to practical scenarios such as recommender systems, clinical trials, and cognitive radio. In these applications,  although the goal is still to maximize the cumulative reward for the overall system, decision-making and observations are naturally distributed at the participating clients.

Several intrinsic challenges arise with the federated contextual bandit formulation. One important issue is that besides regret, we should also take into account the communication cost, which is usually the system bottleneck. To reduce the communication cost while maintaining the same regret guarantee, the clients should transmit the necessary information to the server only when the local information has accumulated to the extent that it would affect the decision-making. Compared with the centralized contextual bandits, which have a linearly growing communication cost, algorithms for federated contextual bandits attempt to achieve a comparable regret with sub-linear communication cost. 

Second, most existing studies on federated contextual bandits focus on the synchronous communication scenario \citep{huang2021federated,li2022communication}, in which all participating clients first upload local information and then download updated global information from the server in each communication round. This stringent communication requirement is often not met in practice. A recent work of \citet{li2022asynchronous} studies the asynchronous federated linear bandit problem. However, communications for different clients are not independent in their approach because the upload from one client may trigger the server to perform downloads for all clients. To address this issue, \citet{he2022simple} proposes FedLinUCB, which enables independent synchronizations between clients and the server.

Third, the majority of prior studies on federated linear bandits focused on the infinite-arm setting \citep{li2022communication,li2022asynchronous,he2022simple} (see \cref{appd:related} for a detailed literature review). From a methodology point of view, these papers largely build on the OFUL principle \citep{abbasi2011improved}. One notable exception is \citet{huang2021federated}, which studies synchronous communication with fixed contexts and proposes the Fed-PE algorithm based on the phased elimination G-optimal design \citep{lattimore2020bandit}. To the best of our knowledge, no prior result exists for federated linear bandits with finite arms and time-evolving adversarial contexts, which is the focus of our work. 

\begin{table*}[tbh]
\vspace{-0.15in}
\caption{Comparison of this paper with related works}
\label{tab:regret_compare}
\begin{center}
\adjustbox{max width= \linewidth}{
\begin{tabular}{ c|c|c|c|c }
\hline
\textbf{System} & \textbf{Action} & \textbf{Algorithm} & \textbf{Regret} & \textbf{Communication}  \\ \hline
Single-player & infinite arm &   {OFUL \citep{abbasi2011improved}} & $d\sqrt{T \log T}$ & N/A  \\ \hline
Single-player & finite fixed arm& {\makecell{PE + G-optimal \\ \citep{lattimore2020bandit} }} & $O(\sqrt{d T \log T})$ & N/A \\ \hline
Single-player & finite adversarial arm& {SupLinUCB \citep{chu2011contextual} }& $O(\sqrt{d T \log^3 T})$ & N/A \\ \hline
{Federated  (Async) }& infinite arm& {FedLinUCB \citep{he2022simple} }& $O(d\sqrt{T}\log T)$ & $O(dM^2\log T)$ \\ \hline
{Federated  (Async) }& infinite arm& {Async-LinUCB \citep{li2022asynchronous} }& $O(d\sqrt{T}\log T)$ & $O(dM^2\log T)$ \\ \hline
{Federated  (Sync) }& infinite arm& {DisLinUCB  \citep{wang2019distributed} } & $O(d\sqrt{T}\log^2 T)$ & $O(dM^{3/2})$ \\ \hline
{Federated  (Sync) } & finite fixed arm& {Fed-PE  \citep{huang2021federated} }& $O(\sqrt{d T \log T})$ & $O(d^2 M K\log T)$ \\ \hline
\rowcolor[gray]{0.9} {Federated  (Async)} & finite adversarial arm & {FedSupLinUCB (This work) }& $O(\sqrt{d T \log^3 T})$ & $O(d M^2 \log d \log T )$ \\ \hline
\rowcolor[gray]{0.9} {Federated  (Sync)} & finite adversarial arm& {FedSupLinUCB   (This work) }& $O(\sqrt{d T \log^3 T})$ & $O(d^{3/2} M^{3/2} \log(d)) $ \\ \hline
\end{tabular}
}
\end{center}
\begin{center}
{\scriptsize $d$: the dimension of the unknown parameter, $M$: the number of clients, $K$: the number of finite actions, $T$: the total arm pulls from all clients.}
\end{center}
\vspace{-0.15in} 
\end{table*}

\paragraph{Main contributions.} Our main contributions are summarized as follows.
\begin{itemize}[leftmargin=*]\itemsep=0pt
\item We develop a general federated bandits framework, termed \algg, for solving the problem of federated linear contextual bandits with finite adversarial actions. \alg extends SupLinUCB \citep{chu2011contextual,ruan2021linear} and OFUL \citep{abbasi2011improved}, two important principles in (single-player, centralized) linear bandits, to the federated bandits setting with a carefully designed layered successive screening. 
\item We instantiate \alg with both asynchronous and synchronous client activities. For the former setting, we propose Async-\alg where communication is triggered only when the cumulative local information impacts the exploration uncertainty to a certain extent. We prove that Async-\alg achieves $\tilde{O}(\sqrt{dT})$ regret with $O( d M^2\log d \log T)$ communication cost, which not only reduces the regret by $\sqrt{d}$ compared with previous results on asynchronous federated linear bandits with infinite arms, but also matches the minimax lower bound up to polylog terms, indicating that Async-\alg achieves order-optimal regret. 
\item For synchronous communications, we propose Sync-\algg, which has a refined communication design where only certain layers are communicated, as opposed to the complete information. Sync-\alg achieves order-optimal regret of $\tilde{O}(\sqrt{dT})$ with horizon-independent communication cost $O(\sqrt{d^3M^3}\log d)$. Compared with the best previous result \citep{huang2021federated} which achieves the same order-optimal regret but only for \emph{fixed} actions, we show that it is the \emph{finite} actions that fundamentally determines the regret behavior in the federated linear bandits setting. 
\item We further develop two extensions of \algg: (1) Variance-adaptive \algg, for which a total regret of $\Tilde{O}(\sqrt{d \sum \nolimits_{t=1}^{T} \sigma_t^2})$ is achieved, where $\sigma_t^2$ is the noise variance at round $t$. (2) Adversarial corruption \algg, for which a total regret of $\tilde{O}(\sqrt{dT} + d C_p)$ is achieved, where $C_p$ is the total corruption budget.
\end{itemize}

\section{Related Works}
\label{appd:related}
% \textcolor{red}{Ruida to add}\\
The linear bandit model, as a generalization of finite armed bandits with linear contextual information, has been extensively studied. The setting of infinite arm sets solved by LinUCB was analyzed in  \citep{dani2008stochastic,abbasi2011improved}, which achieves regret $\tilde{O}(d \sqrt{T})$ with appropriate confidence width \citep{abbasi2011improved} and matches the lower bound \citep{dani2008stochastic} up to logarithmic factors. In contrast, algorithms like SupLinRel \citep{auer2002using} and SupLinUCB \citep{chu2011contextual} achieve $\tilde{O}(\sqrt{dT})$ in the setting of finite time-varying adversarial arm sets under $K \ll 2^d$, with a lower bound $\Omega(\sqrt{dT})$ \citep{chu2011contextual}. The SupLinUCB algorithm was later optimized and matches the lower bound up to iterated logarithmic factors in \citet{li2019nearly}. As a special case of the finite arm setting, if the arm set is time-invariant, an elimination-based algorithm \citep{lattimore2020bandit} via G-optimal design can be applied to achieve similar optimal performance. 

The federated linear bandits problems were studied under the settings of infinite arm set \citep{dubey2020differentially, li2020federated, li2022asynchronous} and time-invariant finite arm set \citep{huang2021federated}, while the time-varying finite arm set setting has not been well explored. A finite time-varying arm set has many meaningful practical applications such as recommendation system \citep{li2010contextual,chu2011contextual}, and the distributed (federated) nature of the applications naturally falls in the federated linear bandits problem with finite time-varying arms. The paper fills this gap by generalizing the SupLinUCB algorithm to the federated setting. 

We study both the asynchronous setting \citep{li2022asynchronous} \citep{he2022simple}, where clients are active on their own and full participation is not required, and the synchronous setting \citep{shi2021federated, dubey2020differentially}, where all the clients make decisions at each round and the communication round requires all the clients to upload new information to the server and download the updated information. We design algorithms so as to reduce the communication cost while maintaining optimal regret. Technically, the communication cost is associated with the algorithmic adaptivity, since less adaptivity requires fewer updates and thus fewer communication rounds. The algorithmic adaptivity of linear bandits algorithms was studied in the single-player setting \citep{han2020sequential} \citep{ruan2021linear}.  It was also considered in the federated setting \citep{wang2019distributed,huang2021federated,salgia2023distributed}.

\section{System Model and Preliminaries}
\label{sec:model}
\subsection{Problem Formulation}
We consider a federated linear contextual bandits model with $K$ finite but possibly time-varying arms. The model consists of $M$ clients and one server in a star-shaped communication network. Clients jointly solve a linear bandit problem by collecting local information and communicating with the central server through the star-shaped network in a federated manner, with no direct communications among clients. The only function of the server is to aggregate received client information and to send back updated information to clients. It cannot directly play the bandits game.

Specifically, some clients $I_t \subseteq [M]$ are active at round $t$. Client $i \in I_t$ receives $K$ arms (actions to take) associated with contexts $\{x_{t, a}^{i} \}_{a \in [K]} \subset \mathbb{R}^d$ with $\|x_{t,a}^{i} \|_2 \leq 1$. Here we adopt the oblivious adversarial setting, where all contexts are chosen beforehand, and not dependent on previous game observation. Client $i$ then pulls an arm $a_t^i \in [K]$ based on the information collected locally as well as previously communicated from the server. A reward $r_{t,a_t^i}^{i} = \theta^{\top} x_{t,a_t^i}^{i} + \epsilon_{t}$ is revealed privately to client $i$, where $\theta \in \mathbb{R}^d$ is an unknown weight vector with $\|\theta\|_2 \leq 1$ and $\epsilon_t$ is an independent $1$-sub-Gaussian noise. At the end of round $t$, depending on the communication protocol, client $i$ may exchange the collected local information with the server so that it can update the global information.

We aim to design algorithms to guide the clients' decision-making and overall communication behaviors. We analyze two patterns of client activity. 1) \textbf{Synchronous}: all $M$ clients are active at each round. 2) \textbf{Asynchronous}: one client is active at each round. For the latter case, we further assume that client activity is independent of data and history. Denote by $T_i$ the number of times client $i$ is active. In the former case, $T_i = T_j, \forall i, j \in [M]$, while in the latter case, $T_i$ may be different among clients. We define $T = \sum_{i=1}^M T_i$ as the total number of arm pulls from all clients. 

The performance is measured under two metrics -- \emph{total regret} and \emph{communication cost}, which concern the decision-making effectiveness and the communication efficiency respectively. Denote by $P_{T}^{i}= \{t \in [T]\ |\ i \in I_t \}$ the set of time indices at which client $i$ is active, with $|P_{T}^{i}| = T_i$. The total regret is defined as 
\begin{align}
 R_T  = \sum_{i = 1}^M R_T^{i}  = \sum_{i = 1}^M \mathbb{E} \left[ \sum \nolimits_{t \in P_{T}^{i}}  r^i_{t, a_t^{i, *}}- r^i_{t, a_t^i} \right],
\end{align}
where $a_t^{i,*} = \arg \max_{a \in [K]} \theta^\top x_{t,a}^i$.  %denotes the best arm at round $t$ for client $i$. 
We define the communication cost as the total number of communication rounds between clients and the server.

\subsection{Preliminaries}
\paragraph{Information encoding.} In the linear bandits setting (federated or not), the information a client acquires is usually encoded by the gram matrix and the action-reward vector. Specifically, when the client has observed $n$ action-reward pairs $\{(x_t, r_t)\}_{t = 1}^n$, the information is encoded by matrix $A_n = \sum_{t=1}^n x_t x_t^\top $ and vector $b_n = \sum_{t=1}^n r_t x_t$. Denote by $\enc(\cdot)$ this encoding function, i.e.,
$A_n, b_n \leftarrow \enc( \{x_t, r_t \}_{t = 1}^n )$. 

\paragraph{Communication criterion.} Communication in our proposed framework is data-dependent, in the same spirit as the ``doubling trick'' introduced in \citet{abbasi2011improved} to reduce the computation complexity in single-player linear bandits. The key idea is that communication is triggered only when the cumulative local information, represented by the determinant of the gram matrix $A_n$, affects the exploration uncertainty to a great extent and hence the client needs to communicate with the server. Detailed communication protocols will be presented in each algorithm design.

% \paragraph{Communication criterion.} %Based on the determinant of matrix $A_n$, 
% %The determined-base \textcolor{red}{CT: what does this mean?} 
% Communication in our proposed FedSupLinUCB-type algorithms is a data-dependent update design, which is first introduced in \cite{abbasi2011improved} known as the "doubling trick" to reduce computation complexity in the single-player linear bandit. Such a criterion is also efficient in federated linear bandits to control communication costs in a multi clients system since with less decision-making policy change, less information exchange is required. The criterion focuses on the determinant of the gram matrix $A_n$ in information encoding, which controls the exploration uncertainty. When the communication criterion is met, it denotes that cumulative local information affects the exploration uncertainty to some extent and the need for clients to communicate with the server so that it can benefit other clients. Moreover, We use a refined communication criterion previously proposed by \cite{wang2019distributed} in the Synchronous Fed-SupLinUCB and get an independent of $T$ communication cost.

\paragraph{Synchronization procedure.} Denote by $\sync()$ a routine that $n$ clients (client 1, $\ldots$, client $n$) first communicate their local gram matrices and action-reward vectors to the server, and the server then aggregates the matrices (and vectors) into one gram matrix (and action-reward vector) and transmits them back to the $n$ clients. Specifically, each client $i$ holds newly observed local information $(\Delta A^i, \Delta b^i)$, which is the difference between the client's current information $(A^i, b^i)$ and the information after the last synchronization. In other words, $(\Delta A^i, \Delta b^i)$ is the information that has not been communicated to the server. The server, after receiving the local information $\{(\Delta A^i, \Delta b^i) \}_{i=1}^n$, updates the server-side information $(A^{ser}, b^{ser})$ by $ A^{ser} \leftarrow A^{ser} + \sum \nolimits_{i=1}^n \Delta A^i, b^{ser} \leftarrow b^{ser} + \sum \nolimits_{i=1}^n \Delta b^i$ and sends them back to each of the $n$ clients. Each client $i$ will then update the local information by $A^i \leftarrow A^{ser}, b^i \leftarrow b^{ser}$. The procedure is formally presented in \cref{alg:protocal-sync}.

\begin{algorithm}[!htb]
   \caption{$\sync$($s$, server, client $1$, $\ldots$ client $n$)}
   \label{alg:protocal-sync}
\begin{algorithmic}[1]
    \For{$i = 1, 2, \ldots, n$} \Comment{Client-side local information upload}
    \State Client $i$ sends the local new layer $s$ information $(\Delta  A^{i}_s, \Delta b^i_s)$ to the server
    \EndFor 
    \State Update server's layer $s$ information:  \Comment{Server-side information aggregation and distribution}
    $$ A_s^{ser} \leftarrow A_s^{ser} + \sum   \nolimits_{i=1}^n \Delta A_s^i, \quad b_s^{ser} \leftarrow b_s^{ser} + \sum \nolimits_{i=1}^n \Delta b_s^i $$   
    \State Send server information $A_s^{ser}, b_s^{ser}$ back to all clients
    \For{$i = 1, 2, \ldots, n$} 
    \State $A_s^i \leftarrow A_s^{ser}, b_s^i \leftarrow b_s^{ser},
         \Delta A_s^i \leftarrow 0, \Delta b_s^i \leftarrow 0$  \Comment{Client $i$ updates the local information}
    \EndFor
\end{algorithmic}
\end{algorithm}

\section{The \alg Framework}
In this section, we present a general framework of federated bandits for linear bandits with finite oblivious adversarial actions. Two instances (asynchronous and synchronous) of this general framework will be discussed in subsequent sections.

\paragraph{Building block: SupLinUCB.} As the name suggests, the proposed \alg framework is built upon the principle of SupLinUCB \citep{chu2011contextual,ruan2021linear}. 
% A detailed discussion is presented in \cref{appd:alg}. 
% \label{appd:alg}
The information $(A, b)$ is useful in the sense that the reward corresponding to an action $x$ can be estimated within confidence interval $x^\top \hat{\theta} \pm \alpha \|x\|_{A^{-1}}$, where $\hat{\theta} = A^{-1} b$. It is shown in \citet{abbasi2011improved} that in linear bandits (even with an infinite number of actions) with $\alpha = \tilde{O}(\sqrt{d})$, the true reward is within the confidence interval with high probability. Moreover, if the rewards in the action-reward vector $b$ are mutually independent, $\alpha$ can be reduced to $O(1)$. The former choice of $\alpha$ naturally guarantees $\tilde{O}(d\sqrt{T})$ regret. However, to achieve regret $\tilde{O}(\sqrt{dT})$, it is critical to keep $\alpha = O(1)$. This is fulfilled by the SupLinUCB algorithm \citep{chu2011contextual} and then recently improved by \citet{ruan2021linear}. The key intuition is to successively refine an action set that contains the optimal action, where the estimation precision of sets is geometrically strengthened. Specifically, the algorithm maintains $(S+1)$ layers of information pairs $\{(A_s, b_s)\}_{s = 0}^S$, and the rewards in the action-reward vectors are mutually independent, except for layer $0$. The confidence radius for each layer $s$ is $w_s = 2^{-s} d^{1.5}/\sqrt{T}$. 
% We set the parameters for controlling confidence width as { $\alpha_{0} = 1 + \sqrt{d \ln(2M^2 T/\delta)}, \alpha_{s} \leftarrow 1 + \sqrt{2 \ln(2 K M T \ln d/\delta)}, \forall s \in [1:S]$} for our federated setting.

\begin{algorithm}[!htb]
    \caption{\texttt{S-LUCB}}
    \label{alg:S-LUCB}
\begin{algorithmic}[1]
    \State \textbf{Initialization}: $S = \lceil \log d \rceil $, 
    $\overline{w}_{0} = d^{1.5}/\sqrt{T}$, $\overline{w}_{s} \leftarrow 2^{-s}\overline{w}_{0}, \forall s \in [1:S]$. \\
    $\alpha_{0} = 1 + \sqrt{d \ln(2M^2 T/\delta)}, \alpha_{s} \leftarrow 1 + \sqrt{2 \ln(2 K M T \ln d/\delta)}, \forall s \in [1:S]$
    \State \textbf{Input}: Client $i$ (with local information $A^i, b^i$, $\Delta A^i, \Delta b^i$), contexts set $\{x_{t,1}^i, \ldots, x_{t,K}^i\}$ 
    \State  $A_{t, s}^i \leftarrow A_s^i + \Delta A_{s}^i, b_{t, s}^i \leftarrow b_s^i + \Delta b_{s}^i \;\;$ 
    or $\;\; A_{t, s}^i \leftarrow A_s^i, b_{t, s}^i \leftarrow b_s^i$ for \texttt{lazy update} 
    \State $\hat{\theta}_{s} \leftarrow (A^i_{t, s})^{-1} b^{i}_{t, s}$, $\hat{r}_{t,s, a}^{i} = \hat{\theta}_s^\top x_{t, a}^i$, $ w_{t, s, a}^i \leftarrow \alpha_s \|x^i_{t,a}\|_{(A^i_{t, s})^{-1}}$, $\forall s \in [0 : S], \forall a \in [K]$
    \State $s \leftarrow 0$; $\mathcal{A}_{0} \leftarrow \{ a\in [K] \mid \hat{r}_{t,0, a}^{i} + w_{t,0, a}^{i} \geq \max_{a \in [K]} (\hat{r}_{t,0, a}^{i} - w_{t,0, a}^{i}) \}$   \Comment{Initial screening}
    \Repeat \Comment{Layered successive screening}
    \If{$s = S$}
    \State 	Choose action $a_t^i$ arbitrarily from $\mathcal{A}_{S}$
    \ElsIf{$w_{t,s, a}^{i} \leq \overline{w}_{s}$ for all $a \in \mathcal{A}_{s}$}
    \State
	$\mathcal{A}_{s+1} \leftarrow \{a \in \mathcal{A}_{s} \mid \hat{r}_{t,s,a}^{i} \geq    
	 \max_{a' \in \mathcal{A}_{s}} (\hat{r}_{t,s, a'}^{i} )-2 \overline{w}_{s} \}$; $s\leftarrow s+1$
    \Else
    \State $a_t^i \leftarrow \arg \max _{ \{a \in \mathcal{A}_{s},w_{t,s,a}^{i}>\overline{w}_s \}} w_{t,s,a}^{i}$
    \EndIf
    \Until{action $a_t^i$ is found}
    \State Take action $a_t^i$ and and receive reward $r_{t,a_t^i}^{i}$
    \State $\Delta A_{s}^{i} \leftarrow \Delta A_{s}^{i} + x_{t,a_t^i}^{i}x_{t,a_t^i}^{i \top}$, $\Delta b_{s}^{i} \leftarrow \Delta b_{s}^{i} + r_{t,a_t^i}^{i} x_{t,a_t^i}^{i}$ \Comment{Update local information}
    \State \textbf{Return} layer index $s$
\end{algorithmic}
\end{algorithm}

\paragraph{FedSupLinUCB.} 
\texttt{S-LUCB}, presented in \cref{alg:S-LUCB}, combines the principles of SupLinUCB  and OFUL \citep{abbasi2011improved} and is the core subroutine for \algg. We maintain $S = \lceil \log d \rceil$ information layers, and the estimation accuracy starts from $d^{1.5}/ \sqrt{T}$ of layer $0$ and halves as the layer index increases. Finally, it takes $\Theta(\log d)$ layers to reach the sufficient accuracy of $\sqrt{d/T}$ and achieves the minimax-optimal regret.

When client $i$ is active, the input parameters $(A^i, b^i)$ contain information received from the server at the last communication round, and $(\Delta A^i, \Delta b^i)$ is the new local information collected between two consecutive communication rounds. $\{x_{t,1}^i, \ldots, x_{t,K}^i\}$ is the set of contexts observed in this round. Client $i$ can estimate the unknown parameter $\theta$ either with all available information or just making a lazy update. This choice depends on the communication protocol and will be elaborated later. During the decision-making process, client $i$ first makes arm elimination at layer $0$ to help bootstrap the accuracy parameters. Then, it goes into the layered successive screening in the same manner as the SupLinUCB algorithm, where we sequentially eliminate suboptimal arms depending on their empirical means and confidence widths. After taking action $a_t^i$ and receiving the corresponding reward $r_{t,a_t^i}^i$, client $i$ updates its local information set $(\Delta A_s^i, \Delta b_s^i)$ by aggregating the context into layer $s$ in which we take the action, before returning layer $s$.

\section{Asynchronous \alg}
In the asynchronous setting, only one client is active in each round. %After observing the context set, the active client chooses the action by \texttt{S-LUCB}, receives a reward, and then decides whether communication with the server is needed by the criterion in Line 7 of \cref{alg:AsynFedSupLinALgo}. %When communication is initiated, the active client exchanges information with the server. 
Note that global synchronization and coordination are not required, and all inactive clients are idle. 

\subsection{Algorithm}
% In Async-\alg presented in \cref{alg:AsynFedSupLinALgo}, we 
% We first initialize the information matrices for all clients and the server (gram matrix and action-reward vector) in each layer $s \in [0:S]$. Without loss of generality, we assume only one client $i_t$ is active at round $t$, i.e, if multiple clients are active, we queue them up and activate them in turn. More discussion of this equivalence can be found in \citet{he2022simple,li2022asynchronous}. The active client chooses the action, receives a reward, updates local information matrices of layer $s$ with a lazy update according to \texttt{S-LUCB}, and decides whether communication with the server is needed by the criterion in Line 7 of \cref{alg:AsynFedSupLinALgo}. If communication is triggered, we synchronize client $i_t$ with the server by \cref{alg:protocal-sync} in \cref{appd:alg}.  

We first initialize the information for all clients and the server (gram matrix and action-reward vector) in each layer $s \in [0:S]$. We assume only one client $i_t$ is active at round $t$. It is without loss of generality since if multiple clients are active, we can queue them up and activate them in turn. More discussion of this equivalence can be found in \citet{he2022simple,li2022asynchronous}. The active client chooses the action, receives a reward, updates local information matrices of layer $s$ with a lazy update according to \texttt{S-LUCB}, and decides whether communication with the server is needed by the criterion in Line 7 of \cref{alg:AsynFedSupLinALgo}. If communication is triggered, we synchronize client $i_t$ with the server by \cref{alg:protocal-sync}.

\begin{algorithm}[!htb]
    \caption{Async-\alg}
    \label{alg:AsynFedSupLinALgo}
\begin{algorithmic}[1]
    \State \textbf{Initialization}: $T$, $C$, $S = \lceil \log d \rceil$ 
    \State $\{ A_{s}^{ser} \leftarrow I_{d}, b_{s}^{ser} \leftarrow 0 \mid s \in [0:S] \}$ \Comment{Server initialization}
    \State $\{ A_{s}^{i} \leftarrow I_{d},\Delta  A_{s}^{i} , b_{s}^{i}, \Delta b_{s}^{i} \leftarrow 0 \mid s \in [0:S], i\in [M]  \}$ \Comment{Clients initialization}
    \For{$t=1,2,\cdots,T $}
    \State Client $i_t = i$ is active, and observes $K$ contexts $\{x_{t, 1}^{i}, x_{t, 2}^{i}, \cdots, x_{t, K}^{i}\}$
    \State $s \leftarrow$ \texttt{S-LUCB} $\left( \text{client }i, \{x_{t, 1}^{i}, x_{t, 2}^{i}, \cdots, x_{t, K}^{i}\} \right)$  with lazy update
    \If{{$ \frac{\det(A_{s}^{i}+\Delta A_{s}^{i})}{ \det(A_{s}^{i}) } > (1+C) $}}
    \State $\sync$($s$, server, clients $i$) for each $s \in [0:S]$
    \EndIf 
    \EndFor
\end{algorithmic}
\end{algorithm}

% \begin{algorithm}[!htb]
%         \LinesNumbered  
% 	\caption{Asynchronous FedSupLinUCB Algorithm}
% 	\label{alg:AsynFedSupLinALgo}
% 	\KwIn{ $T$, $C$, $S = \lceil \log d \rceil$}
%     Server initialization: 
%     $\{ A_{s}^{ser} \leftarrow I_{d}, b_{s}^{ser} \leftarrow 0 \mid s \in [S] \}$\\
%     Clients initialization: 
%      $\{ A_{s}^{i} \leftarrow I_{d},\Delta  A_{s}^{i} \leftarrow 0, b_{s}^{i} \leftarrow 0 ,\Delta b_{s}^{i} \leftarrow 0 \mid s \in [S], i\in [M]  \}$\\
%     \For{$t=1,2,...,T $}{
%     Client $i_t = i$ is active, and observes $K$ contexts $\{x_{t, 1}^{i}, x_{t, 2}^{i}, \cdots, x_{t, K}^{i}\}$\\
%     $s = $\texttt{S-LUCB}(client $i$, $\{x_{t, 1}^{i}, x_{t, 2}^{i}, \cdots, x_{t, K}^{i}\}$) with lazy update\\ 
% 	\If{ \textcolor{blue}{$ \frac{\det(A_{s}^{i}+\Delta A_{s}^{i})}{ \det(A_{s}^{i}) } > (1+C) $}} {
%     Sync($s$, server, clients $i$) for each $s \in [0:S]$\\
%     }
% }
% \end{algorithm} 

\subsection{Performance Analysis}
\begin{theorem}
\label{thm:async}
For any $0<\delta<1$, if we run \cref{alg:AsynFedSupLinALgo} with $C = 1/M^2$, then with probability at least $1 - \delta$, the regret of the algorithm is bounded as  $R_T \leq \tilde{O}\left(\sqrt{d \sum_{i = 1}^MT_i}\right) = \tilde{O}\left(\sqrt{d T} \right)$. Moreover, the corresponding communication cost is bounded by $O( d M^2\log d \log T)$.
\end{theorem}

\noindent \textit{Remark 1.}
The minimax lower bound of the expected regret for linear contextual bandits with $K$ adversarial actions is $\Omega(\sqrt{dT})$, given in \citet{chu2011contextual}. Theorem \ref{thm:async} indicates that Async-\alg achieves order-optimal regret (up to polylog term) with $O(dM^2 \log d \log T)$ communication cost. To the best of our knowledge, this is the first algorithm that achieves the (near) optimal regret in federated linear bandits with finite adversarial actions.

\noindent \textit{Remark 2.} 
Without any communication, each client would execute SupLinUCB \citep{chu2011contextual} for $T_i$ rounds locally, and each client can achieve regret of order $\tilde{O}(\sqrt{dT_i})$. Therefore, the total regret of $M$ clients is upper bound by 
$ R_T \leq \sum_{i=1}^{M} \sqrt{d T_i} \text{ polylog}(T) \leq  \sqrt{dM \sum_{i=1}^{M} T_i} \text{ polylog}(T),$
where the last inequality becomes equality when $T_i = T_j, \forall i, j \in [M]$. Compared with conducting $M$ independent SupLinUCB algorithms locally, Async-\alg yields an average \textit{per-client gain of $1/\sqrt{M}$}, demonstrating that communications in the federated system can speed up local linear bandits decision-making at clients.

\noindent \textit{Remark 3.}
Most previous federated linear bandits consider the infinite action setting, based on the LinUCB principle \citep{abbasi2011improved}. Async-\alg considers a finite adversarial action setting and has a $\sqrt{d}$ reduction on the regret bound. \texttt{Fed-PE} proposed in \citet{huang2021federated} also considers the finite action setting. However, their action sets are fixed. We generalize their formulation and take into account a more challenging scenario, where the finite action set can be chosen adversarially. The regret order is the same as Fed-PE (ignoring the ploylog term), indicating that it is the \emph{finite} actions as opposed to \emph{fixed} actions that fundamentally leads to the $\sqrt{d}$ regret improvement in the federated linear bandits setting.

\paragraph{Communication cost analysis of \algg.} We sketch the proof for the communication cost bound in \cref{thm:async} in the following, while deferring the detailed proofs for the regret and the communication cost to \cref{appd:async}. 

We first study the communication cost triggered by some layer $s$. Denote by $A_{t,s}^{ser}$ the gram matrix in the server aggregated by the gram matrices uploaded by all clients up to round $t$. Define $T_{n,s} = \min \{t \in [T]| \det(A_{t,s}^{ser})\geq 2^{n} \}$, for each $n \geq 0$. We then divide rounds into epochs $\{T_{n,s},T_{n,s}+1, \cdots,\min(T_{n+1,s}-1, T)\}$ for each $n \geq 0$. The number of communications triggered by layer $s$ within any epoch can be upper bounded by $2(M + 1/C)$ (see \cref{lem:comm_layer}), and the number of non-empty epochs is at most $d \log(1+T/d)$ by \cref{lem:epl}. Since there are $S=\lceil \log d \rceil$ layers and synchronization among all layers is performed once communication is triggered by any layer (Line 8 in Algorithm \ref{alg:AsynFedSupLinALgo}), the total communication cost is thus upper-bounded by $O(d (M+ 1/C) \log d \log T)$. Plugging $C = 1/M^2$ proves the result. 

We note that although choosing a larger $C$ would trigger fewer communications, the final choice of $C = 1/M^2$ takes into consideration both the regret and the communication cost, i.e., to achieve a small communication cost while maintaining an order-optimal regret.

\section{Synchronous \alg}
In the synchronous setting, all clients are active and make decisions at each round. Though it can be viewed as a special case of the asynchronous scenario (clients are active and pulling arms in a round-robin manner), the information update is broadcast to all clients. In other words, the key difference from the asynchronous scenario besides that all clients are active at each round is that when a client meets the communication criterion, \emph{all}  clients will upload local information to the server and download the updated matrices. This leads to a higher communication cost per communication round, but in this synchronous scenario, knowing all clients are participating allows the communicated information to be well utilized by other clients. This is in sharp contrast to the asynchronous setting, where if many other clients are active in the current round, uploading local information to the clients seems unworthy. To mitigate the total communication cost, we use a more refined communication criterion to enable time-independent communication cost.

\subsection{The Algorithm}
The Sync-\alg algorithm allows each client to make decisions by the \texttt{S-LUCB} subroutine. Note that the decision-making is based on all available local information instead of the lazy update in the Async-\alg algorithm. The communication criterion involves the count of rounds since the last communication, which forces the communication to prevent the local data from being obsolete. Some layers may trigger the communication criterion either because the local client has gathered enough new data or due to having no communication with the server for too long. We categorize these layers in the CommLayers and synchronize all the clients with the server. 

\begin{algorithm}[!htb]
\caption{Sync-\alg}
\label{alg:SynFedSupLinALgo}
\begin{algorithmic}[1]
    \State \textbf{Initialization}: $T_c$, $D$, $S = \lceil \log d \rceil$, $t_{last}^s \leftarrow 0, \forall s \in [0:S]$, CommLayers $\leftarrow \emptyset$.
    \State $\{ A_{s}^{ser} \leftarrow I_{d}, b_{s}^{ser} \leftarrow 0 \mid s \in [0:S] \}$ \Comment{Server initialization}
    \State $\{ A_{s}^{i} \leftarrow I_{d}, \Delta A_{s}^{i}, b_{s}^{i}, \Delta b_{s}^{i} \leftarrow 0 \mid s \in [0: S], i\in [M]  \}$\Comment{Clients initialization}
    \For{$t=1,2, \cdots, T_c$}
    \For{$i = 1,2, \cdots,M$}
    \State Client $i_t = i$ is active, and observes $K$ contexts $\{x_{t, 1}^{i}, x_{t, 2}^{i}, \cdots, x_{t, K}^{i}\}$
    \State $s \leftarrow $\texttt{S-LUCB} $\left( \text{client } i, \{x_{t, 1}^{i}, x_{t, 2}^{i}, \cdots, x_{t, K}^{i}\} \right)$
    \If{{$(t-t_{last}^s)\log \frac{\det(A_{s}^{i}+\Delta A_{s}^{i})}{ \det(A_{s}^{i}) } > D$}}
    \State Add $s$ to CommLayers
    \EndIf
    \EndFor
    \EndFor
    \For{$s \in$  CommLayers}
    \State $\sync$($s$, server, clients $[M]$); $t_{last}^s \leftarrow t$, CommLayers $\leftarrow  \emptyset$
    \EndFor
\end{algorithmic}
\end{algorithm}

% \begin{algorithm}[!htb]
%         \LinesNumbered  
% 	\caption{Synchronous FedSupLinUCB Algorithm}
% 	\label{alg:SynFedSupLinALgo}
% 	\KwIn{ $T$, $S = \lceil \log d \rceil$}
%     Server initialization: 
%     $\{ A_{s}^{ser} \leftarrow I_{d}, b_{s}^{ser} \leftarrow 0 \mid s \in [S] \}$\\
%     Clients initialization: 
%      $\{ A_{s}^{i} \leftarrow I_{d},\Delta  A_{s}^{i} \leftarrow 0, 
%      b_{s}^{i} \leftarrow 0, \Delta b_{s}^{i} \leftarrow 0 \mid s \in [S], i\in [M]  \}$\\
%      Set $t_{last}^s = 0, \forall s \in [0:S]$, CommLayers = $\emptyset$.\\
%     \For{$t=1,2,...,T $}{
%     \For{$i=1,2..,M$}{
%     Client $i$ observes $K$ contexts $\{x_{t, 1}^{i}, x_{t, 2}^{i}, \cdots, x_{t, K}^{i}\}$\\
%     $s = $\texttt{S-LUCB}(client $i$, $\{x_{t, 1}^{i}, x_{t, 2}^{i}, \cdots, x_{t, K}^{i}\}$) \\ 
% 	\uIf{ \textcolor{blue}{$(t-t_{last}^s)\log \frac{\det(A_{s}^{i}+\Delta A_{s}^{i})}{ \det(A_{s}^{i}) } > D$}} {
% 	Add $s$ to CommLayers
%     }
%     }
%     \For{$s \in \textnormal{CommLayers}$}{
%     Sync($s$, server, clients $[M]$)\\
%     Set $t_{last}^s = t$, CommLayers = $\emptyset$.
%     }
% }
% \end{algorithm} 

\subsection{Performance Analysis}
\begin{theorem}
\label{thm:sync}
For any $0 < \delta < 1$, if we run \cref{alg:SynFedSupLinALgo} with $D = \frac{T_c \log T_c}{d^2 M}$, with probability at least $1 - \delta$, the regret of the algorithm is bounded as 
$R_T \leq \tilde{O}(\sqrt{d M T_c})$ where $T_c$ is the total per-client arm pulls. Moreover, the corresponding communication cost is bounded by $O(\sqrt{d^3 M^{3}} \log d)$.
\end{theorem}

\noindent \textit{Remark 4.}
Theorem \ref{thm:sync} demonstrates Sync-\alg also achieves the minimax regret lower bound while the communication cost is independent of $T_c$. It is particularly beneficial for large $T_c$. Especially, the number of total rounds in the synchronous scenario is $T = M T_c$, while in the asynchronous setting, we have $T = \sum_{i=1}^M T_i$ rounds.

\paragraph{Communication cost analysis of \texttt{Sync-}\algg.} We sketch the proof for the communication cost bound in \cref{thm:sync} below, while deferring the detailed proofs for the regret and the communication cost to \cref{appd:sync}. 

We call the chunk of consecutive rounds without communicating information in layer $s$ (except the last round) an \emph{epoch}. Information in layer $s$ is collected locally by each client and synchronized at the end of the epoch, following which the next epoch starts. Denoted by $A_{p, s}^{all}$ the synchronized gram matrix at the end of the $p$-th epoch. 
%The set of rounds that at least one client is pulling an arm in layer $s$ can then be divided into multiple consecutive epochs, and we further dichotomize these epochs into good and bad epochs in the following definition. 
% \begin{definition}
% % (\textbf{Good and bad epochs}) 
% Suppose the set of rounds that at least one client is pulling an arm in layer $s$ are divided into $P$ epochs and denoted by $A_{p, s}^{all}, b_{p, s}^{all}$ the synchronized gram matrix and reward-action vector at the end of the $p$-th epoch. $P$ epochs can then be dichotomized into
% $ \Pc^{good}_s \triangleq \left\{p \in [P]: \frac{\det(A_{p, s}^{all})}{\det(A_{p-1, s}^{all})} \leq 2\right\}, \Pc^{bad}_s \triangleq [P] \setminus \Pc^{good}_s$, 
% where $A_{0, s}^{all} \triangleq I$. We say round $t$ is \emph{good} if the epoch containing round $t$ belongs to $\Pc_s^{good}$; otherwise $t$ is \emph{bad}.
% \end{definition}
For any value $\beta > 0$, there are at most $\lceil \frac{T_c}{\beta}\rceil$ epochs that contain more than $\beta$ rounds by pigeonhole principle. If the $p$-th epoch contains less than $\beta$ rounds, then $\log(\frac{\det(A_{p,s}^{all})}{ \det(A_{p-1,s}^{all})}) > \frac{D}{\beta}$ based on the communication criterion and the fact that $\sum_{p = 1}^P \log \frac{\det(A_{p, s}^{all})}{\det(A_{p-1, s}^{all})} \leq R_s = O(d \log(T_c))$ (see \cref{equ:sum_good_chunk}). The number of epochs containing rounds fewer than $\beta$ is at most $O(\lceil \frac{R_s}{D/\beta} \rceil )$. 
Noting that $D = \frac{T_c \log(T_c)}{d^2 M}$, the total number of epochs for layer $s$ is at most $\lceil \frac{T_c}{\beta} \rceil + \lceil \frac{R_s \beta}{D} \rceil = O(\sqrt{\frac{T_cR_{s}}{D}}) = O(\sqrt{d^3 M})$ by taking $\beta = \sqrt{\frac{DT_c}{R_s}}$. The total communication cost is thus upper bounded by $O( S M \sqrt{d^3 M} ) = O( \log(d) \sqrt{d^3 M^3} )$.

\section{Extensions of \alg}
In this section, we extend the \alg algorithm to address two distinct settings in federated systems: scenarios characterized by heterogeneous variances, and those affected by adversarial corruptions.

\subsection{Federated Heteroscedastic Linear Bandits}
\label{sec:varian}
We have so far focused on the federated linear bandits with 1-sub-Gaussian reward noises. In this section, we adapt Async-\alg to the case where the reward noises have \emph{heterogeneous} variances, which extends the \emph{heteroscedastic linear bandits} as studied in \citet{zhou2021nearly,zhou2022computationally} to the asynchronous federated setting, where one client is active at a time. Specifically, the reward noises $\{\epsilon_t\}_{t \in [T]}$ are independent with $|\epsilon_{t}| \leq R, \mathbb{E}[\epsilon_{t}] = 0$ and $\mathbb{E}[\epsilon_t^2] \leq \sigma_t^{2}$, where $\sigma_t$ is known to the active client. 

We propose a variance-adaptive Asyc-\alg and analyze its regret and the communication cost in the theorem below, with the algorithm and the proof details in \cref{appd:variance} due to space constraint. The regret is significantly less than that of the Async-\alg when the variances $\{\sigma_t^2\}$ are small. 

% So far, we have focused on the assumption that the reward noise is \emph{sub-Gaussian}. In this section, we extend the analysis of \alg to a special case when the reward noise is further restricted to have bounded magnitude and variance. Specifically, we consider the \emph{heteroscedastic linear bandits} as studied in \citet{zhou2021nearly,zhou2022computationally}. The only difference with the previous problem formulation is that we replace the sub-Gaussian assumption on noise $\epsilon_{t}$ with $|\epsilon_{t}| \leq R, \mathbb{E}[\epsilon_{t}] = 0, \mathbb{E}[\epsilon^2] \leq \sigma_t^{2}$, $\forall t$,  where $\sigma_t$ is a known upper bound of the variance of the noise $\epsilon_{t}$.

% The main results for variance-adaptive \alg are given in \cref{thm:async-variance}. Due to the space limitation, we only present the results for Async-\algg.
\begin{theorem}
\label{thm:async-variance}
For any $0 < \delta < 1$, if we run the variance-adaptive Async-\alg algorithm in \cref{appd:variance} with $C = 1/M^2$, with probability at least $1 - \delta$, the regret is bounded as 
$R_T \leq \tilde{O}(\sqrt{d \sum \nolimits_{t = 1}^T \sigma_t^{2} })$, and the communication cost is bounded by $O(dM^2 \log^2 T)$.
\end{theorem}

% \noindent \textit{Remark 5.} The proof is similar to Async-\alg with sub-Gaussian noise, and the details can be found in \cref{appd:variance}. The main difference is that we apply a new Bernstein-type self-normalized martingale inequality proposed in \citet{zhou2022computationally} for layer $0$ to handle the variance heterogeneity. Our algorithm can achieve regret of order $\tilde{O}(\sqrt{d \sum_{t=1}^{T} \sigma_t^{2}})$ with small communication cost, which achieves the minimax regret lower bound and has a finer dependency on the sum of variances (as opposed to $\sqrt{T}$).

\subsection{Federated Linear Bandits with Corruption}
We further explore asynchronous federated linear bandits with adversarial corruptions, where an adversary inserts a corruption $c_t$ to the reward $r_t$ of the active client at round $t$. The total corruption is bounded by $\sum_{t=1}^T |c_t| \leq C_p$. We incorporate the idea of linear bandits with adversarial corruption studied in \cite{he2022nearly} to the proposed \alg framework and propose the Robust Async-\alg algorithm, with details in Appendix \ref{appd:corruption}. Robust Async-\alg can achieve the optimal minimax regret (matching the lower bound in \cite{he2022nearly}) while incurring a low communication cost. % We further explore the situation of the federated linear bandit with adversarial corruption. We assume there exists an adversary. At round $t$, after observing the reward $r_t$, the adversary inserts an adversarial corruption $c_t$ onto the reward $r_t$. Then, the client receives the final reward $\hat{r}_t = x_t \theta^{\top} +\epsilon_t +c_t$. To measure the level of corruption, we further assume the summation of corruption is bounded as $\sum_{t=1}^{T} |c_t| \leq C_p$.

\begin{theorem}
\label{thm:async-corruption}
For any $0 < \delta < 1$, if we run the Robust Async-\alg algorithm in \cref{appd:corruption} with $C = 1/M^2$, with probability at least $1 - \delta$, the regret is bounded as 
$R_T \leq \tilde{O}(\sqrt{d T } + d C_p )$, and the communication cost is bounded by $O(dM^2 \log d \log T)$.
\end{theorem}
% \noindent \textit{Remark 6.} According to Proposition 4.8 in \citet{he2022nearly}, for any algorithm there exists a corrupted bandits instance with finite adversarial arms such that the algorithm suffers at least $\Omega(\max\{\sqrt{dT}, d C_p\})$ regret. Therefore, our algorithm is near-optimal up to the logarithmic factor.

\section{Experiments}
We have experimentally evaluated \alg in the asynchronous and synchronous settings on both synthetic and real-world datasets. We report the results in this section. 
% Due to the space limitation, we report results from the real-world dataset in \cref{appd:realworld}.
% \subsection{Synthetic Dataset}
% We first conduct simulations to compare the impact of various arrival patterns. Then we compare the federated linear bandit with the decentralized setting, where $M$ clients just pull local linear bandit without communication, and show the improvement in per-client regret and benefit from the federated linear bandit. Moreover, we investigate how well the algorithms balance regrets $R_T$ and communicate cost $C_T$ by tuning parameters $C$ and $D$, which are the corresponding communication thresholds in asynchronous and asynchronous scenarios respectively. 
\subsection{Experiment Results Using Synthetic Dataset}
We simulate the federated linear bandits environment specified in  \cref{sec:model}. With $T = 40000$, $M = 20$, $d = 25$, $\Ac = 20$, contexts are uniformly randomly sampled from an $l_2$ unit sphere, and reward $r_{t,a} = \theta^{\top} x_{t,a}+ \epsilon_t$,  where $\epsilon_t$ is Gaussian distributed with zero mean and variance $\sigma =0.01$.  It should be noted that while $M$ clients participate in each round in the synchronous scenario, only one client is active in the asynchronous case. In the plots, the $x$-axis coordinate denotes the number of arm pulls, which flattens the actions in the synchronous setting. 

\begin{figure*}[tbh]
    % \vspace{-0.05in}
  % \setlength{\abovecaptionskip}{-4pt}
	\centering
	\subfigure[\scriptsize Regret: arrival patterns.]{ \includegraphics[width=0.235\linewidth]{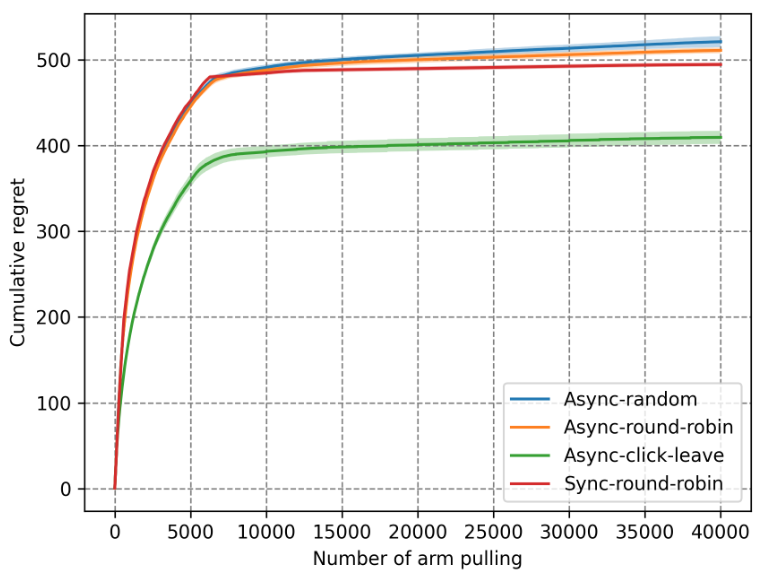}\label{exp:regret_pattern}}
	\subfigure[\scriptsize Communication: arrival patterns.]{ \includegraphics[width=0.235\linewidth]{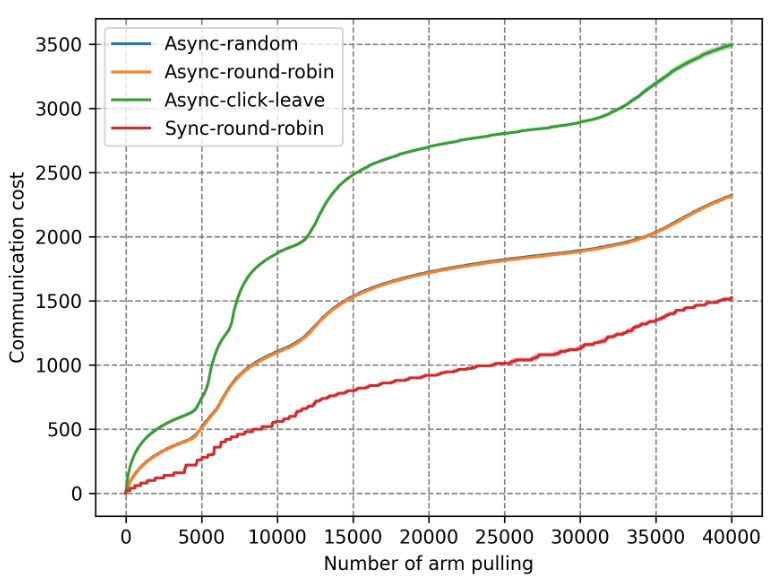}\label{exp:commcost_pattern}}
	\subfigure[\scriptsize Regret: client numbers.]{ \includegraphics[width=0.235\linewidth]{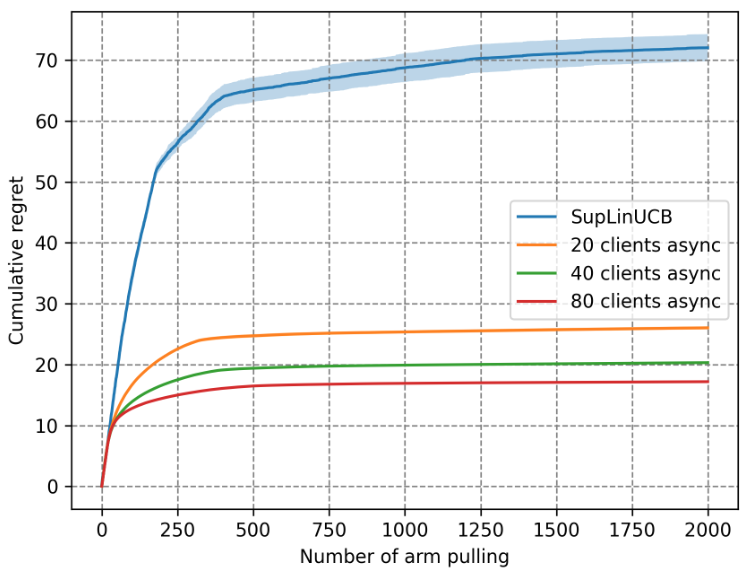}\label{exp:compare_clients}}
	\subfigure[\scriptsize Regret vs communications.]{ \includegraphics[width=0.235\linewidth]{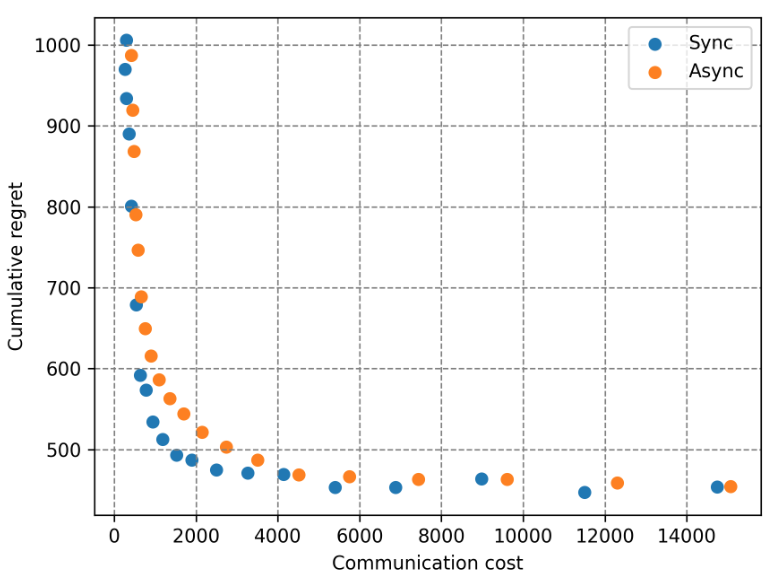}\label{exp:trade-off-synthetic}}
	\caption{\small Experimental results with the synthetic dataset.}
	\label{fig:performance}
        % \vspace{-0.1in}
\end{figure*}

\mypara{Arrival pattern.}
We first investigate the impact of different arrival patterns (the sequence of activating clients): (1) \textbf{Random}, which randomly allocates $T/M$ arm pulls in $[T]$ for each client. (2) \textbf{Round-robin}, i.e. $[1,2,3,\cdots, M, 1,2,3,\cdots M, \cdots]$. (3) \textbf{Click-leave}, i.e. $[1,1,\cdots, 2,2,\cdots, \cdots, M, M, \cdots]$. 
The regret and the communication cost of these three arrival patterns in the synthetic experiment are reported in \cref{exp:regret_pattern} and \cref{exp:commcost_pattern}, respectively. We note that although the upper bound analysis in our proof is for the worst-case instance, the numerical results suggest that different arrival patterns result in diverse regret performances. Round-robin and random patterns are more challenging since both local bandit learning and each client's policy updates happen relatively slowly. The click-leave pattern, which is the closest to the centralized setting, achieves the best regret. In addition, compared with Async-\alg, Sync-\alg achieves better cumulative regrets with a higher communication cost.

% The experiment's findings in figure\ref{exp:regret_pattern} and \ref{exp:commcost_pattern} demonstrate that diverse arrival patterns result in a range of regret performances. While the upper bound analysis in our proof is for the worst-case instance analysis, a tighter bound can be obtained for the analysis of some particular arrival patterns. The round-robin and random patterns are more challenging in terms of regret since we both need to maintain the $M$-client linear bandits and clients' policy updates relatively slowly. The click-leave  pattern, closest to the centralization scenario, achieves the best regret since the clients' data and policy updates are more centralized. In addition, Sync-\alg achieves good cumulative regrets with a higher communication cost as  all the clients take part in each triggered communication, compared with Async-\algg.  For the communication costs, each time communication in Sync-\alg is triggered, all clients are required to respond, which greatly escalates communication costs. Async-\alg with random and round-robin patterns (which almost entirely overlap in the figure) achieves the lowest communication cost. Additionally, as data focuses more on the present client in Async-\alg with the click-leave pattern, communication costs go up.

\mypara{Amount of clients.}
The per-client cumulative regret as a function of $T_c = T/M$ with different amounts of clients is plotted in \cref{exp:compare_clients}. In comparison to the baseline SupLinUCB, \alg algorithms achieve better regret via communication between clients and the server. We can see from the experiment that  \alg significantly reduces the per-client regret compared with SupLinUCB, and achieves a better regret as  $M$ increases in both asynchronous and synchronous settings.

\mypara{Trade-off between regrets and communications.} We evaluate the tradeoff between communication and regret by running \alg with different communication threshold values $C$ and $D$ in asynchronous and synchronous settings respectively. The results are reported in \cref{exp:trade-off-synthetic}, where each scattered dot represents the communication cost and the cumulative regret that \alg has achieved with a given threshold value at round $T = 40000$. We see a clear tradeoff between the regret and the communication. More importantly, Sync-\alg achieves a better tradeoff than Async-\algg.

\subsection{Experiment Results Using Real-world Dataset}
\label{appd:realworld}
We further investigate how efficiently the federated linear bandits algorithm performs in a more realistic and difficult environment. We have carried out experiments utilizing the real-world recommendation dataset MovieLens 20M \citep{harper2015movielens}. Following the steps in \citet{li2022communication}, we first filter the data by maintaining users with above $2500$ movie ratings and treating rating points greater than $3$ as positive, ending up with $N = 37$ users and 121934 total movie rating interactions. Then, we follow the process described in \citet{cesa2013gang} to generate the contexts set, using the TF-IDF feature $d = 25$ and the arm set $K = 20$. We plot the per-client normalized rewards of the FedSupLinUCB algorithm with different client numbers  $M$ in synchronous and asynchronous cases respectively. Note that the per-client cumulative rewards here are normalized by a random strategy. From  \cref{exp:movielnes_async} and \cref{exp:movielens_sync}, we can see that in both synchronous and asynchronous experiments, \alg has better rewards than SupLinUCB, and the advantage becomes more significant as the number of users increases.

\begin{figure*}[tbh]
	\centering
	\subfigure[\scriptsize Async-\algg.]{ \includegraphics[width=0.4\linewidth]{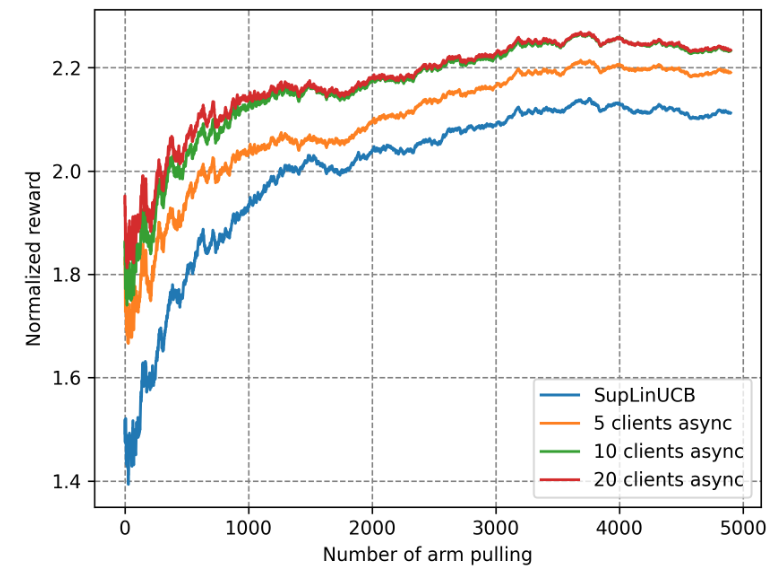}\label{exp:movielnes_async}}
	\subfigure[\scriptsize Sync-\algg.]{ \includegraphics[width=0.4\linewidth]{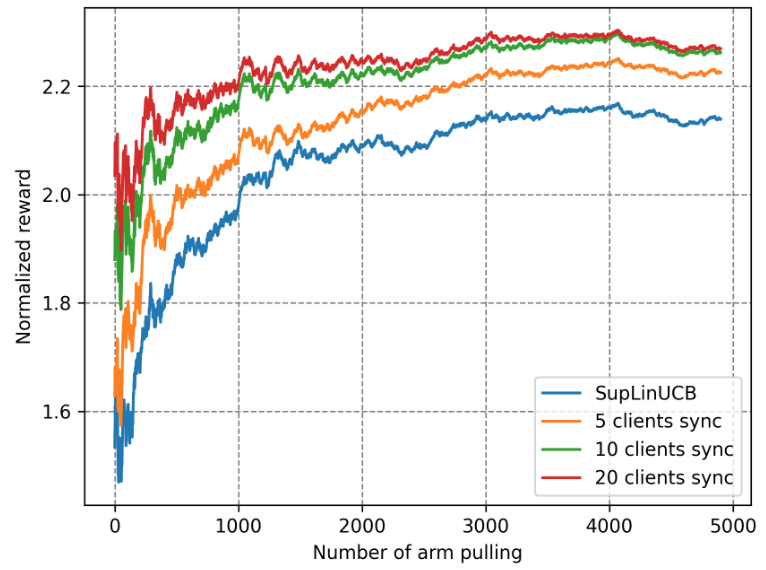}\label{exp:movielens_sync}
	}
	\caption{\small Experimental results with the real-world MovieLens-20M dataset.}
	\label{fig:performance2}
\end{figure*}

% Intuitively, a higher communication cost would benefit the regret -- but how much can be gained and is this tradeoff different for asynchronous and synchronous settings? We answer these questions by running \alg with  different communication threshold values $C$ and $D$ in asynchronous and synchronous settings respectively. The tradeoff results are reported in \cref{exp:trade-off-synthetic}, where each scattered dot represents the communication cost and the cumulative regret that \alg has achieved with a given threshold value at round $T = 80000$. We see a clear tradeoff between regrets and communications and more importantly, Async-\algg achieves a better tradeoff than Sync-\algg. 

% To compare the trade-off between communication cost and cumulative regrets, we run the FedSupLinUCB algorithm with different communication threshold values $C$ and $D$ in asynchronous and synchronous settings respectively. We set client number $M = 40$ and each client participate in $T_0=2000$ rounds. The experiment result is depicted in figure\ref{exp:trade-off-synthetic}, and it should be noted that each dot in the scatter plots represents the communication cost and cumulative regrets that the FedLinUCB algorithm has achieved with a given threshold value at round $T = 80000$. We can see that communication costs rise as cumulative regret decreases in both asynchronous and synchronous cases. In the asynchronous setting, FedSupLinUCB achieves the best cumulative regret with the lowest communication costs, indicating that it is more communication-efficient than the synchronous case.

\section{Conclusion}
We studied federated linear bandits with finite adversarial actions, a model that has not been investigated before. We proposed \alg that extends the SupLinUCB and OFUL principles to the federated setting in both asynchronous and synchronous scenarios, and analyzed their regret and communication cost, respectively. The theoretical results proved that  \alg is capable of approaching the minimal regret lower bound (up to polylog terms) while only incurring sublinear communication costs, suggesting that it is the \emph{finite} actions that fundamentally determines the regret behavior in the federated linear bandits setting. Furthermore, we examined the extensions of the algorithm design to the variance-adaptive and adversarial corruption scenarios. %As a future direction, it would be interesting to see whether we can estimate the variance information when it is not given, and investigate how this affects the regret behavior. 

% In this work, we consider federated linear bandit with finite adversarial actions in a system with $M$ clients and the central server and fill the gap that has not been addressed in previous work. We extend the SupLinUCB algorithm to the federated setting in both asynchronous and synchronous scenarios, achieve a total regret of $\tilde{O(dT)}$ and communication cost of $O(dM\log d\log T )$ and $\sqrt{d^3M^3}\log d$. Our algorithms achieve regret of order optimal and have the $\sqrt{d}$ reduction compared with infinite actions. 

\begin{ack}
The work of LF and CS was supported in part by the U.S. National Science Foundation (NSF) under grants 2143559, 2029978, and 2132700. 
%The work of RZ and CT was supported in part by XXX.

\end{ack}

\newpage
\bibliography{NIPS_ref}

\begin{thebibliography}{}

\bibitem[Abbasi-Yadkori et~al., 2011]{abbasi2011improved}
Abbasi-Yadkori, Y., P{\'a}l, D., and Szepesv{\'a}ri, C. (2011).
\newblock Improved algorithms for linear stochastic bandits.
\newblock {\em Advances in neural information processing systems}, 24.

\bibitem[Auer, 2002]{auer2002using}
Auer, P. (2002).
\newblock Using confidence bounds for exploitation-exploration trade-offs.
\newblock {\em Journal of Machine Learning Research}, 3(Nov):397--422.

\bibitem[Cesa-Bianchi et~al., 2013]{cesa2013gang}
Cesa-Bianchi, N., Gentile, C., and Zappella, G. (2013).
\newblock A gang of bandits.
\newblock {\em Advances in neural information processing systems}, 26.

\bibitem[Chu et~al., 2011]{chu2011contextual}
Chu, W., Li, L., Reyzin, L., and Schapire, R. (2011).
\newblock Contextual bandits with linear payoff functions.
\newblock In {\em Proceedings of the Fourteenth International Conference on Artificial Intelligence and Statistics}, pages 208--214. JMLR Workshop and Conference Proceedings.

\bibitem[Dani et~al., 2008]{dani2008stochastic}
Dani, V., Hayes, T.~P., and Kakade, S.~M. (2008).
\newblock Stochastic linear optimization under bandit feedback.
\newblock {\em 21st Annual Conference on Learning Theory}, pages 355--366.

\bibitem[Dubey and Pentland, 2020]{dubey2020differentially}
Dubey, A. and Pentland, A. (2020).
\newblock Differentially-private federated linear bandits.
\newblock {\em Advances in Neural Information Processing Systems}, 33:6003--6014.

\bibitem[Han et~al., 2020]{han2020sequential}
Han, Y., Zhou, Z., Zhou, Z., Blanchet, J., Glynn, P.~W., and Ye, Y. (2020).
\newblock Sequential batch learning in finite-action linear contextual bandits.
\newblock {\em arXiv preprint arXiv:2004.06321}.

\bibitem[Harper and Konstan, 2015]{harper2015movielens}
Harper, F.~M. and Konstan, J.~A. (2015).
\newblock The {MovieLens} datasets: History and context.
\newblock {\em ACM Trans. Interact. Intell. Syst.}, 5(4):1--19.

\bibitem[He et~al., 2022a]{he2022simple}
He, J., Wang, T., Min, Y., and Gu, Q. (2022a).
\newblock A simple and provably efficient algorithm for asynchronous federated contextual linear bandits.
\newblock {\em arXiv preprint arXiv:2207.03106}.

\bibitem[He et~al., 2022b]{he2022nearly}
He, J., Zhou, D., Zhang, T., and Gu, Q. (2022b).
\newblock Nearly optimal algorithms for linear contextual bandits with adversarial corruptions.
\newblock {\em Advances in neural information processing systems}.

\bibitem[Huang et~al., 2021]{huang2021federated}
Huang, R., Wu, W., Yang, J., and Shen, C. (2021).
\newblock Federated linear contextual bandits.
\newblock {\em Advances in Neural Information Processing Systems}, 34:27057--27068.

\bibitem[Lattimore and Szepesv{\'a}ri, 2020]{lattimore2020bandit}
Lattimore, T. and Szepesv{\'a}ri, C. (2020).
\newblock {\em Bandit algorithms}.
\newblock Cambridge University Press.

\bibitem[Li and Wang, 2022a]{li2022asynchronous}
Li, C. and Wang, H. (2022a).
\newblock Asynchronous upper confidence bound algorithms for federated linear bandits.
\newblock In {\em International Conference on Artificial Intelligence and Statistics}, pages 6529--6553. PMLR.

\bibitem[Li and Wang, 2022b]{li2022communication}
Li, C. and Wang, H. (2022b).
\newblock Communication efficient federated learning for generalized linear bandits.
\newblock {\em arXiv preprint arXiv:2202.01087}.

\bibitem[Li et~al., 2010]{li2010contextual}
Li, L., Chu, W., Langford, J., and Schapire, R.~E. (2010).
\newblock A contextual-bandit approach to personalized news article recommendation.
\newblock In {\em Proceedings of the 19th international conference on World wide web}, pages 661--670.

\bibitem[Li et~al., 2020]{li2020federated}
Li, T., Song, L., and Fragouli, C. (2020).
\newblock Federated recommendation system via differential privacy.
\newblock In {\em 2020 IEEE International Symposium on Information Theory (ISIT)}, pages 2592--2597. IEEE.

\bibitem[Li et~al., 2019]{li2019nearly}
Li, Y., Wang, Y., and Zhou, Y. (2019).
\newblock Nearly minimax-optimal regret for linearly parameterized bandits.
\newblock In {\em Conference on Learning Theory}, pages 2173--2174. PMLR.

\bibitem[McMahan et~al., 2017]{mcmahan2017fl}
McMahan, B., Moore, E., Ramage, D., Hampson, S., and y~Arcas, B.~A. (2017).
\newblock Communication-efficient learning of deep networks from decentralized data.
\newblock In {\em Proc. AISTATS}, pages 1273--1282, Fort Lauderdale, FL, USA.

\bibitem[Ruan et~al., 2021]{ruan2021linear}
Ruan, Y., Yang, J., and Zhou, Y. (2021).
\newblock Linear bandits with limited adaptivity and learning distributional optimal design.
\newblock In {\em Proceedings of the 53rd Annual ACM SIGACT Symposium on Theory of Computing}, pages 74--87.

\bibitem[Salgia and Zhao, 2023]{salgia2023distributed}
Salgia, S. and Zhao, Q. (2023).
\newblock Distributed linear bandits under communication constraints.
\newblock In {\em International Conference on Machine Learning}, pages 29845--29875. PMLR.

\bibitem[Shi et~al., 2021]{shi2021federated}
Shi, C., Shen, C., and Yang, J. (2021).
\newblock Federated multi-armed bandits with personalization.
\newblock In {\em International Conference on Artificial Intelligence and Statistics}, pages 2917--2925. PMLR.

\bibitem[Wang et~al., 2019]{wang2019distributed}
Wang, Y., Hu, J., Chen, X., and Wang, L. (2019).
\newblock Distributed bandit learning: Near-optimal regret with efficient communication.
\newblock In {\em International Conference on Learning Representations}.

\bibitem[Zhou and Gu, 2022]{zhou2022computationally}
Zhou, D. and Gu, Q. (2022).
\newblock Computationally efficient horizon-free reinforcement learning for linear mixture mdps.
\newblock {\em arXiv preprint arXiv:2205.11507}.

\bibitem[Zhou et~al., 2021]{zhou2021nearly}
Zhou, D., Gu, Q., and Szepesvari, C. (2021).
\newblock Nearly minimax optimal reinforcement learning for linear mixture markov decision processes.
\newblock In {\em Conference on Learning Theory}, pages 4532--4576. PMLR.

\end{thebibliography}
\bibliographystyle{apalike}

% \section{Discussions}

% \subsection{Discussion of SupLinUCB}
% \label{appd:alg}

% The information $(A, b)$ is useful in the sense that the reward corresponding to an action $x$ can be estimated within confidence interval $x^\top \hat{\theta} \pm \alpha \|x\|_{A^{-1}}$, where $\hat{\theta} = A^{-1} b$. It is shown in \citet{abbasi2011improved} that in linear bandits (even with an infinite number of actions) with $\alpha = \tilde{O}(\sqrt{d})$, the true reward is within the confidence interval with high probability. Moreover, if the rewards in action-reward vector $b$ are mutually independent, $\alpha$ can be reduced to $O(1)$. The former choice of $\alpha$ naturally guarantees $\tilde{O}(d\sqrt{T})$ regret, while to achieve regret $\tilde{O}(\sqrt{dT})$, it is critical to keep $\alpha = O(1)$. This is fulfilled by the SupLinUCB algorithm \citep{chu2011contextual} and then recently improved by \citep{ruan2021linear}. The idea is to keep the $(S+1)$ layers of information pairs $\{(A_s, b_s)\}_{s = 0}^S$, and the rewards in the action-reward vectors are mutually independent. The confidence radius for each layer $s$ is $w_s = 2^{-s} d^{1.5}/\sqrt{T}$.

\newpage

\appendix

\section{Supporting Lemmas}
\label{appd:lemmas}

\begin{lemma}(Lemma 11 in \citet{abbasi2011improved}) 
\label{lem:epl}
Let $\left\{X_t\right\}_{t=1}^{\infty}$ be a sequence in $\mathbb{R}^d, V$ be a $d \times d$ positive definite matrix, and define $\bar{V}_t=V+\sum_{s=1}^t X_s X_s^{\top}$. Then, we have 
$$
\log \left(\frac{\operatorname{det}\left(\bar{V}_n\right)}{\operatorname{det}(V)}\right) \leq \sum_{t=1}^n\left\|X_t\right\|_{\bar{V}_{t-1}^{-1}}^2 .
$$
Further, if $\left\|X_t\right\|_2 \leq L$ for all $t$, then
$$
\sum_{t=1}^n \min \left\{1,\left\|X_t\right\|_{\bar{V}_{t-1}^{-1}}^2\right\} \leq 2\left(\log \operatorname{det}\left(\bar{V}_n\right)-\log \operatorname{det} V\right) \leq 2\left(d \log \left(\left(\operatorname{trace}(V)+n L^2\right) / d\right)-\log \operatorname{det} V\right).
$$
Finally, if $\lambda_{min}(V) \geq \max(1, L^2)$, then
\begin{align*}
    \sum_{t= 1}^n \|X_t\|^2_{\bar{V}_{t-1}^{-1}} \leq 2 \log \frac{\det(\bar{V}_n)}{\det(V)}.
\end{align*}
\end{lemma}

\begin{lemma}
\label{lem:gramnorm}
(Lemma 12 in \citet{abbasi2011improved}). Let $A, B$, and $C$ be positive semi-definite matrices such that $A=B+C$. Then, we have 
$$
\sup _{x \neq \mathbf{0}} \frac{x^{\top} A x}{x^{\top} B x} \leq \frac{\det(A)}{\det(B)}.
$$
\end{lemma}

\begin{theorem}
\label{the:ellipsoid confidence} 
(Theorem 2 in \citet{abbasi2011improved}). Let $\left\{\mathcal{F}_i\right\}_{i=0}^{\infty}$ be a filtration. Let $\left\{x_i\right\}_{i=1}^{\infty}$ be an $\mathbb{R}^d$-valued stochastic process such that $x_i$ is $\mathcal{F}_{i-1}$-measurable and $\|x_i \| \leq 1$ almost surely. Let $\left \{ \epsilon_i \right \}_{i=1}^{\infty}$ be a real-valued stochastic process such that $\varepsilon_i$ is $\mathcal{F}_i$-measurable and is sub-Gaussian with variance proxy $1$ when conditioned on $\mathcal{F}_{i-1}$. Fix $\theta \in \mathbb{R}^d$ such that $\| \theta \| \leq 1$. Let $A_n=  I+\sum_{i=1}^n x_i x_i^{\top}, r_i= x_i^{\top} \theta+\varepsilon_i$, and $\hat{\theta}_n= A_n^{-1} \sum_{i=1}^n r_i x_i$. For every $\delta>0$, we have 
$$
\mathbb{P}\left[\forall n \geq 0:\left\|\hat{\theta}_n-\theta \right\|_{A_n} \leq 1 +\sqrt{d \ln \left(\frac{1+n}{\delta}\right)}\right] \geq 1-\delta,
$$
where we define $\|x\|_{A} = \sqrt{x^{\top} A x}$. Furthermore, when the above event holds, we have for every $n \geq 0$ and any vector $x \in \mathbb{R}^d$ that
$$
\left| x^{\top} (\hat\theta_n-\theta)\right| \leq\left( 1 +\sqrt{d \ln \left(\frac{1+n}{\delta}\right)}\right) \sqrt{x^{\top} A_n^{-1} x} .
$$
\end{theorem}

\begin{lemma}
(Adapted from Lemma B.1  in \citet{he2022simple})
% \alpha_{0} = 1 + \sqrt{d \ln(2M^2 T/\delta)}, \alpha_{s} \leftarrow 1 + \sqrt{2 \ln(2 K M T \ln d/\delta)}
\label{lem:he_layer0} 
Under the setting of \cref{thm:async}, establish $C = 1/M^2, \alpha_{0} = 1 + \sqrt{d \ln(2M^2 T/\delta)}$. In layer $0$, with probability at least $1-\delta$, the good event $\Ec_{0}$ happens:
$$ \Ec_{0}  \triangleq \left\{ \left|x_{t,a}^{i \top} \hat{\theta}_{t,s}^i - x_{t,a}^{i \top} \theta\right| \leq w_{t,s,a}^{i}, \forall i \in [M], a \in[K], t \in[T], s = 0 \right\}. $$
\end{lemma}

\begin{lemma} 
\label{lem: independent confidence} 
(Lemma 31 in \citet{ruan2021linear}).
Given $\theta, x_1, x_2, \ldots, x_n \in \mathbb{R}^d$ such that $\|\theta\| \leq 1$, for all $i \in [n]$, let $r_i= x_i^{\top} \theta + \epsilon_i$ where $\epsilon_i$ is an independent sub-Gaussian random variable with variance proxy $1$. Let $A =  I+\sum_{i=1}^n x_i x_i^{\top}$, and $\hat{\theta} = A^{-1} \sum_{i=1}^n r_i x_i$. For any $x \in \mathbb{R}^d$ and any $\alpha>0$, we have 
$$
\mathbb{P} \left[ | x^{\top}(\theta- \hat{\theta})|>(\alpha+ 1) 
\Vert x \Vert_{A^{-1}} \right]\leq 2 \exp (-\alpha^2 / 2) .
$$
\end{lemma}

\section{Lemmas for the SupLinUCB Subroutine}
We present several useful lemmas that are based on \cref{alg:S-LUCB}. Recall that $\Psi_{t,s}$ represents the index set of rounds up to and including round $t$ during which an action is taken in layer $s$. That is,
\begin{align*}
\Psi_{t,s} = \{t'\in [t]:\exists i \in [M], \ a_{t'}^{i} \text{ is chosen in  layer } s\}, \forall s \in [0:S].    
\end{align*}

Similar to Lemma 4 in \citet{chu2011contextual}, we claim that the rewards associated with rounds within each  $\Psi_{t,s}, s\in [S]$ (excluding layer $0$) are mutually independent. 

\begin{lemma}
\label{lem:SupLin-independence}
For each $t \in[T]$ each $s \in[S]$, given any fixed sequence of contexts $\{x_{t, a}^{i},t \in \Psi_{t,s} \}$, the rewards $\{r_{t,s, a}^{i},t \in \Psi_{t,s} \}$ are independent random variables with means $\mathbb{E} [r_{t,s,a}^{i} ]= \theta^{\top} x_{t,s, a}^{i}$.
\end{lemma}

\begin{proof}[Proof of \cref{lem:SupLin-independence}]
For each $s \in [S]$ and each time $t$, the procedure of generating $\Psi_{t,s}$ only depends on the information in previous layers $\cup_{\sigma <s} \Psi_{t,\sigma}$ and confidence width $\{ w_{t,s,a}^{i},a\in [K] \}$. From its definition, $w_{t,s,a}^{i}$ only depends on $\{ x_{\tau,a_{\tau}},\tau \in \Psi_{t-1,s} \}$ and on the current context $x_{t,a}^{i}$. Thus the procedure of generating $\Psi_{t,s}$ does not depend on rewards $\{ r_{\tau, a_{\tau}}, \tau \in \Psi_{t-1,s}\}$, and therefore the rewards are independent random variables when conditioned on $\Psi_{t,s}$.
\end{proof}

Given the above-mentioned statistical independence property, and by referring to \cref{lem: independent confidence}, we can establish the following lemma for each layer $s \in [S]$.

\begin{lemma}
\label{lem:con_layers}
Suppose the time index set $\Psi_{t,s}$ is constructed so that for fixed $x_{\tau,a_{\tau}}$ with $\tau \in \Psi_{t,s}$, the rewards $\{r_{\tau,a_{\tau}}\}$ are independent random variables with mean $\mathbb{E}[r_{\tau,a_{\tau}}] = \theta^{\top} x_{\tau,a_{\tau}}  $.
For any round $t \in [T]$, if client $i_{t} = i $ is active and chooses arm $a_t$ in layer $s \in [S]$, then with probability at least $1-\frac{\delta}{MT\ln d}$, we have for any $a_t \in [K]$: 
$$ \left|\hat{r}_{t,s, a_t} -\theta^{\top} x_{t, a_t}^{i }  \right| \leq w_{t,s,a_t}^{i} = \alpha_s 
\| x_{t,a_t}^{i} \|_{(A_{t,s}^{i})^{-1}}.$$
\end{lemma}

For layer $0$, we employ the self-normalized martingale concentration inequality as outlined in \citet{he2022simple}. By resorting to \cref{lem:he_layer0}, we obtain the following:

\begin{lemma}
\label{lem:event_layer0}
For any round $t \in [T]$, given that client $i_{t} = i $ is active in round $t$ and arm $a_t$ is chosen in layer $0$, with probability at least $1-\delta$, we have for any $a_t \in [K]$: 
$$ \left|\hat{r}_{t,0, a_t} -\theta^{\top} x_{t, a_t}^{i} \right| \leq w_{t,0,a_t}^{i} = \alpha_0
\| x_{t,a_t}^{i} \|_{(A_{t,0}^{i})^{-1}}. $$
\end{lemma}

Summarizing the discussions presented in \cref{lem:con_layers} and \cref{lem:event_layer0}, we now proceed to define the following good event:

\begin{lemma}
\label{lem:goodevent}
Define the good event $\Ec$ as: 
\begin{align}
\Ec  \triangleq \left\{\left| \hat{r}_{t,s, a} - x_{t,a}^{i \top} \theta\right| \leq w_{t,s,a}^{i}, \forall i \in [M], a \in[K], t \in[T], s\in [0:S] \right\}. 
\end{align}
We have $\Pb [\Ec] \geq 1 - \delta$.
\end{lemma}

Conditioned on the good event $\Ec$, the ensuing lemma illustrates that the optimal arm persists in the candidate set, and that the regret experienced in each layer aligns with the order of the confidence width. 

\begin{lemma}
\label{lem:best_arm}
Conditioned on the good event $\Ec$, for $t \in[T]$, assume that client $i$ is active and chooses an action $a_t \in \Ac_s$, and recall $(a_t^{i})^*$ represents the optimal arm in the current round. For any $s' \leq s$, we have: 
$$ (a_t^{i})^* = \argmax_{a \in [K]} \theta^{\top} x_{t,a}^{i} = \argmax_{a\in \Ac_{s'}} \theta^{\top} x_{t,a}^{i}. $$
\end{lemma}

\begin{proof}[Proof of \cref{lem:best_arm}]
For any time step $t \in [T]$, when the good event $\Ec$ holds, by the arm elimination rule in layer $0$, we have 
$$ \hat{r}_{t,0,a^*} +w_{t,0,a^*} \geq  \max_{a\in[K]} \theta^{\top} x_{t,a} \geq \max_{a\neq a^*} \theta^{\top} x_{t,a} \geq \max_{a\neq a^*} ( \hat{r}_{t,0,a} -w_{t,0,a}).$$
Thus, $a^* \in \Ac_0$. For each layer $s'< s$, we have:
$$ \hat{r}_{t,s',a^*} +w_{t,s',a^*} \geq  \max_{a\in \Ac_{s'}} \theta^{\top} x_{t,a} \geq \max_{a\neq a^*,a\in\Ac_{s'}} \theta^{\top} x_{t,a} \geq \max_{a\neq a^*, a \in \Ac_{s'}} (\hat{r}_{t,s',a} - w_{t,s',a}) .$$
Thus, we derive $\hat{r}_{t,s',a^*} \geq \max_{a \in \Ac_{s'}} (\hat{r}_{t,s',a}) - 2 \overline{w}_{s'}$, which follows from $w_{t,s',a} \leq \overline{w}_{s'}$ for all $a \in \Ac_{s'}$ by the arm elimination rule in Line 10 \cref{alg:S-LUCB}. Therefore, arm eliminations will preserve the best arm.
\end{proof}

The forthcoming lemma demonstrates that, under the good event, the regret experienced in each layer aligns with the order of the corresponding confidence width.

\begin{lemma}
\label{lem:regret_lay}
Conditioned on the good event $\Ec$, for $t \in [T]$ client $i \in [M]$ and $s \in [S]$, it holds that: 
\begin{align}
&\mathbb{I}[a_t \text{ is chosen in layer } 0] (\max_{a \in \Ac_0} \theta^{\top} x_{t,a} - \theta^{\top} x_{t,a_t} ) \leq 4 w_{t,0,a_t}, \\
&\mathbb{I}[a_t \text{ is chosen in layer } s] (\max_{a \in \Ac_s} \theta^{\top} x_{t,a} - \theta^{\top} x_{t,a_t} ) \leq 8 \overline{w}_{s}.
\end{align}
\end{lemma}

\begin{proof}[Proof of \cref{lem:regret_lay}]
If an action is taken in layer $0$, we have that
$$ a_t = \argmax_{a \in \Ac_0, w_{t,0,a} > \overline{w_{0}}} w_{t,0,a}, $$
and 
\begin{align*}
\max_{a \in \Ac_0} \theta^{\top} (x_{t,a} - \theta^{\top} x_{t,a_t})
& \leq \max_{a \in \Ac_0}\theta^{\top} x_{t,a} -\min_{a \in \Ac_0}\theta^{\top} x_{t,a}\\
& \leq \max_{a \in \Ac_0} (\hat{\theta}_0^{\top} x_{t,a}+w_{t,0,a}) -\min_{a \in \Ac_0} (\hat{\theta}_0^{\top} x_{t,a}-w_{t,0,a})\\
& \leq 4 \max_{a \in \Ac_0} w_{t,0,a} \\
& = 4 w_{t,0,a_t}.
\end{align*}
The second inequality is conditioned on the good event $\Ec$, and the third inequality arises from the arm elimination rule. If an action is taken in layer $s$, we establish the following:
$$ a_t = \argmax_{a \in \Ac_s, w_{t,s,a} > \overline{w}_s} w_{t,s,a}, $$
and 
\begin{align*}
\max_{a \in \Ac_s} (\theta^{\top} x_{t,a} - \theta^{\top} x_{t,a_t})
& \leq \max_{a \in \Ac_{s-1}} (\hat{\theta}_{s-1}^{\top} x_{t,a}+w_{t,s-1,a}) -\min_{a \in \Ac_{s-1}} (\hat{\theta}_{s-1}^{\top} x_{t,a}-w_{t,s-1,a})\\
& \leq 2 \max_{a \in \Ac_{s-1}} w_{t,s-1,a} + \max_{a \in \Ac_{s-1}} \hat{\theta}_{s-1}^{\top} -\min_{a \in \Ac_{s-1}} \hat{\theta}_{s-1}^{\top} x_{t,a}\\
& \leq 2 \max_{a \in \Ac_{s-1}} w_{t,s-1,a} +2\overline{w}_{s-1}\\
& \leq 4 \overline{w}_{s-1} \leq 8\overline{w}_{s}.
\end{align*}
The first inequality is based on the good event $\Ec$, the third inequality follows the arm elimination rule, and the fourth inequality is due to $w_{t,s-1,a} \leq \overline{w}_{s-1}$ for all $a \in \Ac_{s-1}$.
% and the last inequality is because we only choose the action with confidence width greater than $\overline{w}_s$.
\end{proof}

\section{Supporting Lemmas and Proofs for Async-\algg}
\label{appd:async}

\begin{lemma}
(Lemma 6.2 in \citet{he2022simple})
\label{lem:comm_layer}
In any epoch from round $T_{n,s}$ to round $T_{n+1,s}-1$, the number of communications is at most $2(M+1/C)$.
\end{lemma}

\paragraph{Proof outline of Async-\algg.} 
% First, we reorder the arrival pattern, show the reordered system is equivalent to the original system, and introduce the necessary definitions for the analysis. Second, we use a virtual global model that contains the information of all clients up to round $t$ and then connect the local models with the global model. Finally, we bound the regret and communication cost in each layer $s \in [0:S]$ before summing to the total regret and communication costs, respectively.

First, we reorganize the arrival pattern, demonstrating that the rearranged system parallels the original system, and present the requisite definitions for our analysis. Second, we deploy a virtual global model encapsulating information about all clients up to round $t$, subsequently interconnecting the local models with this global model. Lastly, we derive upper bounds on the regret and communication cost in each layer $s \in [0:S]$ prior to aggregating them to yield the total regret and communication costs, respectively.

% \paragraph{Proof of the regret bound in \cref{thm:async}.} 
Suppose that client $i$ communicates with the server at rounds $t_1, t_2$ with $t_1< t_2$ and does not communicate during the rounds in between. The actions and information gained by client $i$ at the rounds $t_1 < t < t_2$ do not impact other clients' decision-making, since the information is kept local without communication. Therefore, we can reorder the arrival of clients appropriately while keeping the reordered system equivalent to the original system. 

More specifically, suppose client $i$ communicates with the server at two rounds $t_m$ and $t_n$ and does not communicate in the rounds in between (even if she is active). We reorder all the active rounds of client $i$ in $t_m < t < t_n$ and place them sequentially after the round $t_m$. Hence, the arrival of clients can be reordered such that each client communicates with the server and keeps active until the next client's communication begins.  We assume that the sequence of communication rounds in the reordered arrival pattern is $0 = t_{0} < t_{1} < t_{2} < \cdots < t_{N} = T$, where in rounds $t_i \leq t < t_{i+1}$, the active client is the same. Details of the reordering process are given in \cref{def:reorder_func}. Due to the equivalence between the original system and the reordered system,  we carry out the proofs in the reordered system. Note that only one client $i_t$ is active at round $t$, we will write $a_t = a_t^{i_t}$, $x_t = x^{i_t}_{t, a_t}$ and $r_t = r_{t, a_{t}}^i$ for simplicity.

\begin{definition}
\label{def:inform}
\textbf{Client information.} Recall for each client $i \in[M]$, we denote by $L_{i}(t)$ the last round when client $i$ communicated with the server before and including round $t$. E.g., $L_{i}(t)=t$ if client $i$ communicates at round $t$. For each round $t$ each client $i$ and each layer $s$, the information that has been uploaded by client $i$ to the  server is: $ A_{t,s}^{i,up} = \sum \nolimits_{t'=1}^{L_{i}(t)} x_{t'} x_{t'}^{\top} \mathbb{I}\{i_{t'}= i, a_t \text{ in layer } s\}, b_{t,s}^{i,up} = \sum \nolimits_{t' = 1}^{L_{i}(t)} r_{t'} x_{t'} \mathbb{I}\{i_{t'}= i, a_t \text{ in layer } s\} $, 
and the local information in the buffer that has not been uploaded to the server is: $\Delta A_{t,s}^{\mathrm{i}} = \sum \nolimits_{t'= L_{i}(t)+1}^{t} x_{t'} x_{t'}^{\top} \mathbb{I}\{i_{t'}= i, a_t \text{ in layer } s\}, \Delta b_{t,s}^{i} = \sum \nolimits_{t'= L_{i}(t)+1}^{t} r_{t'} x_{j} \mathbb{I}\{i_{t'}= i, a_t \text{ in layer } s\}$. 

\textbf{Server information.} The information in the server is the data uploaded by all clients up to round $t$: $A_{t,s}^{ser} = I + \sum_{i=1}^{M} A_{t,s}^{i,up},  b_{t,s}^{ser} = \sum_{i=1}^{M} b_{t,s}^{i,up}$. 

\textbf{Time index set.} Denote by $\Psi_{t,s}$ the time index set when the action $a_t^{i}$ is chosen in layer $s$. It can be expressed as $\Psi_{t,s} = \{t'\in [t], a_{t'}^{i} \text{ in layer } s,\ i \in [M] \}, s \in [0:S] \}$.

\textbf{Virtual global information.} We define a virtual global model that contains all the information up to round $t$ as: $A_{t,s}^{all} = I +\sum_{t' \in \Psi_{t,s} } x_{t'} x_{t'}^{\top},  b_{t,s}^{all}=\sum_{t' \in \Psi_{t,s}} r_{t'} x_{t'}$.
\end{definition}

The information that is stored on the server and all the information that has not yet been uploaded by clients are combined to generate the global information: $A_{t,s}^{all} = A_{t,s}^{ser} + \sum_{i=1}^{M} \Delta A_{t,s}^{i}, b_{t,s}^{all} = b_{t,s}^{ser} +\sum_{i=1}^{M} \Delta b_{t,s}^{i}$.

Before presenting the proof, we define \emph{good event} $\Ec$ as 
$$\Ec  \triangleq \left\{ \left|x_{t,a}^{i \top} \hat{\theta}_{t,s}^i - x_{t,a}^{i \top} \theta\right| \leq w_{t,s,a}^{i}, \forall i \in [M], a \in[K], t \in[T], s\in [0:S]  \right\}.$$ 
Recall $ \hat{\theta}_{t,s}^i$ is the estimate of $\theta$ by client $i$, and $x_{t,a}^{i}$ and $w_{t,s,a}^{i}$ is the corresponding context and confidence width of the action taken at round $t$. The following lemma shows the good event happens with high probability, similar to the result in \cref{lem:goodevent}. 
\begin{lemma}
\label{lem:goodevt}
It holds that $\Pb [\Ec] \geq 1 - \delta$.
\end{lemma}
 
Conditioned on the good event, to upper bound the regret, we bound the confidence width in each layer via the size of each time index set in the lemma below.
\begin{lemma}
\label{lem:width_lay}
Conditioned on the good event $\Ec$, for each $s \in [0: S-1]$ we have:
\begin{align*}
\sum_{t \in \Psi_{T,s}} w_{t,s,a}^{i} \leq \alpha_s \sqrt{2(1+MC)} \sqrt{2 d  |\Psi_{T,s}| \log |\Psi_{T,s}|} +\alpha_s d M\log(1+T/d).
\end{align*}
\end{lemma}

Noting that $|\Psi_{T,s}| \leq T$ naturally holds, we give a tighter (dimension-dependent) bound on the size of $\Psi_{T,0}$ so as to mitigate the larger coefficient $\alpha_0$ as follows.
\begin{lemma}
\label{lem:Psi_0_bound}
The size of $\Psi_{T,0}$ can be bounded by $|\Psi_{T,0}| \leq T \log T \log(2MT/\delta) /d$.
\end{lemma}

We postpone the proofs of \cref{lem:width_lay} and \cref{lem:Psi_0_bound} until the end of this section, and instead focus on presenting the regret analysis next. Equipped with the previous lemmas, we are ready to analyze the total regret.
\begin{proof}[Proof of \cref{thm:async}](\textbf{Regret analysis}) 
The total regret can be decomposed w.r.t. layers as follows: 
$$R_T  = \mathbb{E}  \sum_{t \in \Psi_{T, 0}} (r^i_{t,a_t^{i, *}}- r^i_{t,a_t}) + \sum_{s=1}^{S} \mathbb{E} \sum_{t \in \Psi_{T,s}} (r^i_{t,a_t^{i, *}}- r^i_{t,a_t}).$$

Conditioned on the good event $\Ec$, we first bound the regret in layer $0$ by 
 \begin{align*}
& \mathbb{E} \sum_{t \in \Psi_{T, 0}} (r^i_{t,a_t^{i, *}}- r^i_{t,a_t}) \leq \sum_{t \in \Psi_{T, 0}} 4 w_{t,0,a_t} \\
& \leq  4 \alpha_0 \sqrt{2(1+MC)} \sqrt{2 d  |\Psi_{T,0}| \log |\Psi_{T,0}|} + 4 \alpha_0 d M\log(1+T/d) s
\leq \tilde{O}( \sqrt{(1+MC) dT}).
\end{align*}

The first inequality follows Lemma~\ref{lem:regret_lay}, the second inequality is from  Lemma~\ref{lem:width_lay}, and the last inequality is due to Lemma~\ref{lem:Psi_0_bound}. We next bound the regret in each layer $s\in [1:S-1]$  similarly by
\begin{align*}
& \sum_{t \in \Psi_{T,s}} \mathbb{E} \left[ r^i_{t, a_t^{i, *}}- r^i_{t,a_t} \right] \leq \sum_{t \in \Psi_{T,s}} 8 \overline{w}_{s} \leq \sum_{t \in \Psi_{T,s}} 8 w_{t,s,a_t} \\
& \leq  8 \alpha_s \sqrt{2(1+MC)} \sqrt{2 d  |\Psi_{T,s}| \log |\Psi_{T,s}|} + 8 \alpha_s d M\log(1+T/d) 
\leq \tilde{O}( \sqrt{(1+MC) dT})
\end{align*}

where the first inequality follows \cref{lem:regret_lay}, the second inequality is from the arm selection rule in line 13 \cref{alg:S-LUCB}, and the third inequality is from \cref{lem:width_lay}. For the last layer $S$, we have: 
\begin{align*}
\sum_{t \in \Psi_{T,S}} \mathbb{E} \left[ r^i_{t,a_t^{i, *}}- r^i_{t,a_t} \right] \leq \sum_{t \in \Psi_{T,S}} 8 \overline{w}_{S} \leq 8 \overline{w}_{S} |\Psi_{T,S}|\leq 8\overline{w}_{S} T  \leq 8\sqrt{dT}.
\end{align*}
Finally, with Lemma~\ref{lem:goodevt}, we have $ R_{T} \leq \tilde{O}( \sqrt{(1+MC) dT})$.

(\textbf{Communication cost analysis}) Next, we study the communication cost in an asynchronous setting. For each layer $s$, $i \geq 0$, we define $ T_{n,s} = \min \{t \in [T]| \det(A_{t,s}^{ser})\geq 2^{i} \}. $
We divide rounds in each layer into epoch $\{T_{n,s},T_{n,s}+1,..,\min(T,T_{n+1,s}-1) \}$, and the communication rounds in the epoch $T_{n,s} \leq t \leq T_{n+1,s}-1$ can be bound by \cref{lem:comm_layer}. Let $N'$ be the largest integer such that $T_{N',s}$ is not empty. According to Lemma \ref{lem:epl} that $\log (\det(A_{t,s}^{all}) ) \leq d \log(1+|\Psi_{T,s}|/d)$, $N'\leq d \log(1+T/d)$. The total number of epochs of layer $s$ is bounded by $d \log(1+T/d)$. By lemma \ref{lem:comm_layer} the communication rounds in layer $s$ is bounded by $O((M+1/C))d \log T$. There are $S=\lceil \log d \rceil$ in the FedSupLinUCB algorithm, the total communication cost is thus upper bound by $O(d (M+ 1/C) \log d \log T)$. Plugging in $C = 1/M^2$ proves the result.

\end{proof}

\begin{definition}
\label{def:reorder_func}
(\textbf{Reorder function}) 
Without loss of generality, we assume all clients communicate with the server at round $t_0=0$, and the sequence of rounds that clients communicate with the server in the original system is $0\leq t_{0}<t_1 <t_2 <...<t_N \leq T$. Define $I_{t, i} = \mathbb{I}(\text{client $i$ communicates with the server at round $t$} )$. Denote by $L_i(t)$ the last communication round of client $i$ before and including round $t$:
\begin{align*}
   L_i(t):= \inf\{u: \sum_{t'=0}^{u} I_{t', i} = \sum_{t'=0}^{t} I_{t', i} \}.
\end{align*}
Denote by $N_i(t)$ the next communication round of client $i$ including and after round $t$:
\begin{align*}
    N_i(t) := \inf\{u: \sum_{t'= t}^{u} I_{t', i} =1 \}.
\end{align*}
The round $t \in [T]$ in the original system is placed in round $\phi(t)$ by the reordering function $\phi: [T] \rightarrow [T]$. 
We first reorder the communication round, suppose two consecutive communication rounds $t_{n}$ and $t_{n+1}$ with $t_{n}<t_{n+1}$, and client $i$ is active at round $t_{n}$ and client $j$ is active at round $t_{n+1}$.
\begin{equation*}
  \phi(t_{n+1}) =
      \phi(t_{n}) + \sum \nolimits_{t'=t_{n}}^{N_i(t_n)} I(i_{t'}= i) -1.     
\end{equation*}
Then we reorder the no-communication rounds, assuming client $i$ is active at round $t$ and does not communicate at this round. We first find the last communication round of client $i$ as $L_i(t)$, and place round $t$ by $\phi(t)$:
\begin{equation*}
  \phi(t) =
      \phi(L_i(t)) + \sum \nolimits_{t'=L_i(t)}^{t} I(i_{t'}= i) -1.
\end{equation*}
\end{definition}

\begin{lemma} 
\label{lem:ser_local}
(Adapted from Lemma 6.5 in \citet{he2022simple}) For each round $t \in [T]$ each layer $s \in [0: S]$  and each client $i \in[M]$ , we have: 
\begin{align*}
A_{t,s}^{ser} =  I + \sum_{i=1}^{M} A_{t,s}^{i,up} \succeq \frac{1}{C} \Delta  A_{t,s}^{i}.
\end{align*}
Further averaging the inequality above over $M$ clients, we have:
\begin{align*}
A_{t,s}^{ser}= I+\sum_{i=1}^{M} A_{t,s}^{i,up} \succeq \frac{1}{M C} \sum_{i=1}^{M} \Delta A_{t,s}^{i}.    
\end{align*}
\end{lemma}

\begin{proof}[Proof of \cref{lem:ser_local}]
Without loss of generality, we consider client $i$ and fix any round $t \in[T]$. Let $t_{1} \leq t$ be the last round such that client $i$ was active at round $t_{1}$. If client $i$ communicated with the server at round $t_{1}$, and chose action $a_{t_1}$ at layer $s$, then we have
$$
A_{t,s}^{ser} = I +\sum_{i=1}^{M} A_{t,s}^{i,up} \succeq  \frac{1}{C} \Delta A_{t_{1},s}^{i} = 0
$$
for other layers $s' \neq s$, according to the determinant-based communication criterion, we have:
$$
\det(A_{t_1,s'}^{i}+\Delta A_{t_1,s'}^{i}) < (1+C) \det(A_{t_1,s'}^{i}).
$$
By \cref{lem:gramnorm} we have
$$
A_{t, s'}^{i} = A_{t_{1},s'}^{i} \succeq \frac{1}{C} \Delta A_{t_{1},s'}^{i}.
$$
Otherwise, if no communication happened at round $t_{1}$, by the communication criterion, at the end of round $t_{1}$, for each layer $s \in [0:S]$, we have $A_{t_{1},s}^{i} \succeq \frac{1}{C} \Delta 
A_{t_{1},s}^{i}$.
Note that $\{ A_{t_{1},s}^{i},\ s\in [0:S] \}$ are the downloaded gram matrices from last communication before round $t_{1}$, so it must satisfy $A_{t_{1},s}^{i} \preceq A_{t_{1},s}^{ser}$ for all $s \in [0: S]$.
For round $t$, since client $i$ is inactive from round $t_{1}$ to $t$, we have for all $s \in [0:S]$:
$$
A_{t,s}^{ser} \succeq A_{t_{1},s}^{ser} \succeq A_{t_{1},s}^{i} \succeq \frac{1}{C} \Delta A_{t_{1},s}^{i} = \frac{1}{C} \Delta A_{t,s}^{i}  
$$
where the last equality holds for inactivation, which completes the proof of the first claim. Further average the above inequality over all clients $i \in[M]$, and we get:
$$
A_{t,s}^{ser}= I+\sum_{i=1}^{M} A_{t,s}^{i,up} \succeq \frac{1}{M C} \sum_{i=1}^{M} \Delta A_{t,s}^{i}.
$$
\end{proof}

Recall that client $i$ utilizes $A_{t,s}^{i}$ and $b_{t,s}^{i}$ to make the decision at round $t$, which were received from the server during the last communication. The following lemma establishes a connection between the gram matrix of the virtual global model and the gram matrix in the active client at round $t$.

\begin{lemma}
\label{lem:loc_ser_gram}
 In the reordered arrival pattern, for any $1 \leq t_{1}<t_{2} \leq T$, suppose client $i$ communicates with the server at round $t_{1}$, and keep active during rounds $t_{1} \leq t \leq t_{2}-1$. Then for rounds $t_{1}+1 \leq t \leq t_{2}-1$, it holds that for each $s \in [0:S] $:
$$
A_{t,s}^{i} \succeq \frac{1}{1+M C} A_{t,s}^{all} .
$$
\end{lemma}

\begin{proof}[Proof of \cref{lem:loc_ser_gram}]
Client $i$ is the only active client from round $t_{1}$ to $t_{2}-1$ and  only communicated  with the server at round $t_{1}$, which implies that for $t_{1}+1 \leq t \leq t_{2}-1 \forall s\in [0:S]$, we have
$$
A_{t,s}^{i} =  I +\sum_{i=1}^{M} A_{t_{1},s}^{i,up} = I + \sum_{i=1}^{M} A_{t,s}^{i,up} \succeq \frac{1}{1+MC} (I + \sum_{i=1}^{M} A_{t,s}^{i,up} + \sum_{i=1}^{M} \Delta A_{t,s}^{i} ) \succeq \frac{1}{1+MC} A_{t,s}^{all} 
$$
where the second equality holds due to the fact that no clients communicate with the server from round $t_{1}+1$ to $t_{2}-1$, and the first inequality follows \cref{lem:ser_local}. 
\end{proof}

\begin{proof}[Proof of \cref{lem:width_lay}]
\label{proof:lem_width_lay}
For $t \in \Psi_{T,s}$, if no communication happened at round $t$, under \cref{lem:loc_ser_gram} and \cref{lem:gramnorm}, we can connect confidence width at the local client with the global gram matrix as:  $$ \Vert x_{t,a}^{i_{t}} \Vert_{(A_{t,s}^{i_{t}})^{-1}} \leq \sqrt{1+MC} \Vert x_{t,a}^{i}\Vert_{(A_{t,s}^{all})^{-1}}.$$
It remains to control the communication rounds in $\Psi_{T,s}$. We define
$$
T_{n}=\min \left\{t \in \Psi_{T,s} \mid \operatorname{\det}\left({A}_{t,s}^{ {all }}\right) \geq 2^{n}\right\},
$$
and let $N'$ be the largest integer such that $T_{N'}$ is not empty. According to \cref{lem:epl}, we have:
$$
\log (\det(A_{t,s}^{all}) ) \leq d \log(1+|\Psi_{T,s}|/d).
$$
Thus, $N'\leq d \log(1+T/d)$. For each time interval from $T_{n}$ to $T_{n+1}$ and each client $i \in[M]$, suppose client $i$ communicates with the server more than once, and communication rounds sequentially are $T_{n, 1}, T_{n, 2}, \ldots, T_{n, k} \in\left[T_{n}, T_{n+1}\right)$. Then for each $j=2, \ldots, k$, since client $i$ is active at rounds $T_{n, j-1}$ and $T_{n, j}$, we have
$$
\|x_{T_{n, j}} \|_{(A_{T_{n, j},s}^{i})^{-1}} \leq \|x_{T_{n, j}} \|_{(A_{T_{n, j-1}+1,s}^{i})^{-1}} \leq \sqrt{1+MC}\|x_{T_{n, j}}\|_{ ((A_{T_{n, j-1}+1,s}^{all})^{-1}}.
$$
Since $\det ({A}_{T_{n+1}-1,s}^{all}) / \det ({A}_{T_{n, j-1}+1,s}^{all }) \leq 2^{n+1} /2^{n} = 2$, by the definition of $T_{n}$, we have:
$$
\| x_{T_{n, j}}\|_{(A_{T_{n, j},s}^{i})^{-1}} \leq \sqrt{2(1+MC)} \|x_{T_{n, j}}\|_{(A_{T_{n+1}-1,s}^{all})^{-1}} \leq  \sqrt{2(1+MC)} \|x_{T_{n, j}} \|_{ (A_{T_{n, j},s}^{all} )^{-1}},
$$
where the second inequality comes from $A_{T_{n+1}-1,s}^{all} \succeq
A_{T_{n, j},s}^{all}$. Specifically, for round $T_{i, 1}$ the first communication round, we can bound the confidence width by $1$.
Thus, for the communication rounds in $\Psi_{T,s}$, we have:
$$
\sum_{t\in \Psi_{T,s}, round\ t\ comm} \Vert x_{t,a}^{i_{t}} \Vert_{(A_{t,s}^{i_{t}})^{-1}} 
\leq MN'+ \sum_{t\in \Psi_{T,s}, round\ t\ comm} \sqrt{2(1+MC)} \Vert x_{t,a}^{i_{t}} \Vert_{(A_{t,s}^{all})^{-1}}.
$$
Finally, we put all rounds in $\Psi_{T,s}$ together:
$$
\begin{aligned}
\sum_{t\in \Psi_{T,s}} w_{t,s,a}
&= \alpha_s \sum_{t\in \Psi_{T,s}} \Vert x_{t,a}^{i_{t}} \Vert_{(A_{t,s}^{i_{t}})^{-1}} \\
&\leq \alpha_s \sum_{t\in \Psi_{T,s}} \sqrt{2(1+MC)} \Vert x_{t,a}^{i_{t}} \Vert_{(A_{t,s}^{all})^{-1}} + \alpha_s MN'\\ 
&\leq \alpha_s \sqrt{2(1+MC)} \sqrt{2 d  |\Psi_{T,s}| \log |\Psi_{T,s}|} +  \alpha_s d M\log(1+T/d)
\end{aligned}
$$
where the second inequality follows \cref{lem:epl}.
\end{proof}

\begin{proof}[Proof of \cref{lem:Psi_0_bound}]
Based on the algorithm, if we choose an action in layer $0$, the selected arm is
$$
a_t = \arg \max_{ a \in \Ac_{0},w_{t,0,a} > \overline{w}_0 } w_{t,0,a},
$$
and the corresponding confidence width satisfies $w_{t,0,a_t} > \overline{w}_0$. Furthermore, 
\begin{align*}
\overline{w}_0 |\Psi_{T,0}|
& \leq \sum_{t \in \Psi_{T, 0}} w_{t,0,a_t}  = \alpha_0 \sum_{t \in \Psi_{T, 0}} \|x_{t,a_t}\|_{(A_{t,0}^{i_t})^{-1}} \\
& \leq \alpha_0 \sqrt{2(1+MC)} \sqrt{2 d  |\Psi_{T,s}| \log |\Psi_{T,s}|} +  \alpha_0 d M\log(1+T/d),
\end{align*}
where the last inequality is by \cref{lem:width_lay}. We can thus conclude that $\Psi_{T,0} \leq T \log T \log(2MT/\delta) /d$.

\end{proof}

\section{Supporting Lemmas and Proofs for Sync-\algg}
\label{appd:sync}

\paragraph{Proof outline of Sync-\algg.} 
To prove a high-probability regret bound, we first define the good event $\Ec$ in the following lemma, under which the regret bound is derived.
\begin{lemma}
Define $\Ec \triangleq \{\left|x_{t,a}^{i \top} \hat{\theta}_{t,s}^i - x_{t,a}^{i \top} \theta\right| \leq w_{t,s,a}^{i}, \forall i \in [M], a \in[K], t \in[T_c], 0 \leq s \leq S \}$.
Then, $\Pb [\Ec] \geq 1 - \delta$.
\end{lemma}

Define client $i$'s one-step regret at round $t$ as
$\ssf{reg}_t^i = \theta^\top( x^i_{t, a^{i*}_t} - x^i_{t, a_t})$. Let $\ssf{reg}_{t, s}^i = \ssf{reg}_t^i$ if action $a_t$ is chosen in layer $s$; otherwise $\ssf{reg}_{t, s}^i = 0$. The total regret can be written as
$$R_T = \sum_{i = 1}^{M} \sum_{t = 1}^{T_c} \ssf{reg}_t^i = \sum_{s = 0}^{S} \sum_{i = 1}^M \sum_{t = 1}^{T_c} \ssf{reg}_{t,s}^i.$$ Fix an arbitrary $s \in \{0, 1, \ldots, S\}$, we analyze the total regret induced by the actions taken in layer $s$, i.e., $R_{s,T_c} = \sum_{i = 1}^M \sum_{t = 1}^{T_c} \ssf{reg}_{t,s}^i$. The analysis can be carried over to different $s$ in the same manner.

We call the chunk of consecutive rounds without communicating information in layer $s$ (except the last round) an \emph{epoch}. In other words, information in layer $s$ is collected locally by each client and synchronized at the end of the epoch, following which the next epoch starts. The set of rounds that at least one client is pulling an arm in layer $s$ can then be divided into multiple consecutive epochs, and we further dichotomize these epochs into good and bad epochs in the following definition. 
\begin{definition}
(\textbf{Good epoch}) Suppose the set of rounds that at least one client is pulling an arm in layer $s$ are divided into $P$ epochs and denoted by $A_{p, s}^{all}, b_{p, s}^{all}$ the synchronized gram matrix and reward-action vector at the end of the $p$-th epoch. $P$ epochs can then be dichotomized into
$ \Pc^{good}_s \triangleq \left\{p \in [P]: \frac{\det(A_{p, s}^{all})}{\det(A_{p-1, s}^{all})} \leq 2\right\}, \Pc^{bad}_s \triangleq [P] \setminus \Pc^{good}_s$, 
where $A_{0, s}^{all} \triangleq I$. We say round $t$ is \emph{good} if the epoch containing round $t$ belongs to $\Pc_s^{good}$; otherwise $t$ is \emph{bad}.
\end{definition}

We bound regrets in layer $s$ induced by the good and bad epochs separately in the following lemmas. Recall $\Psi_{t,s}$ is the time index set when the action $a_t^{i}$ is chosen in the $s$ layer.
\begin{lemma}
\label{lem:syn_good}
Conditioned on the good event $\Ec$, for each layer $s\in [0:S]$, the regret induced by good epochs of layer $s$ is bounded as 
$\sum_{ t \in \Psi_{T_c, s}, t \text{ is good}} \ssf{reg}_{t, s}^i \leq \tilde{O}\left(\alpha_{s} \sqrt{ d |\Psi_{T_c,s}| \log(MT_c) } \right)$.
\end{lemma}

\begin{lemma}
\label{lem:syn_bad}
Define $D = \frac{T_c \log T_c}{d^2 M}$ and $R_s = d \log(1+\frac{|\Psi_{T_c,s}|}{d}) $. 
Conditioned on the good event $\Ec$, for each layer $s\in [0:S]$, the regret induced by bad epochs of layer $s$ is bounded as $
    \sum_{t \in \Psi_{T_c, s}, t \text{ is bad}} \ssf{reg}_{t, s}^i \leq  O\left(\alpha_s M \sqrt{D} R_s \right)$.
\end{lemma}

\begin{lemma}
\label{lem:Syn_Psi_bound}
We have 
$
|\Psi_{T_c,s}| \leq \tilde{O}( \frac{M T_c}{d} )$. 
\end{lemma}

\begin{proof}[Proof of \cref{thm:sync}] (\textbf{Regret analysis}) For each $s \in [0:S]$, the regret induced in layer $s$ is bounded by:
\begin{align*}
R_{s,T_c} &\leq \sum_{ t \in \Psi_{T_c, s}, t \text{ is good}} \ssf{reg}_{t,s}^{i} +\sum_{ t \in \Psi_{T_c, s}, t \text{ is bad}}\ssf{reg}_{t,s}^{i} \\
&\leq O( \alpha_s \sqrt{d |\Psi_{T_c, s}| \log(MT)} + \alpha_s M \sqrt{D} R_s) 
\leq \tilde{O}(\sqrt{d M T_c})
\end{align*}
where the second inequality is from Lemmas \ref{lem:syn_good} and \ref{lem:syn_bad}, and the last inequality is due to \cref{lem:Syn_Psi_bound}.  The total regret can thus be bounded as $R_{T} = \sum_{s = 0}^S R_{s,T_c} =  \tilde{O}(\sqrt{d M T_c})$. 
\end{proof}
% (\textbf{Communication cost analysis})
% For any value $\beta > 0$, there are at most $\lceil \frac{T_c}{\beta}\rceil$ epochs that contain more than $\beta$ rounds. If the $p$-th epoch contains less than $\beta$ rounds, then $\log(\frac{\det(A_{p,s}^{all})}{ \det(A_{p-1,s}^{all})}) > \frac{D}{\beta}$ by the communication criterion and  $\sum_{p = 1}^P \log \frac{\det(A_{p, s}^{all})}{\det(A_{p-1, s}^{all})} = \log \det(A_{P,s}^{all}) \leq d \log(1+\frac{|\Psi_{T_c,s}|}{d})$, where the last inequality  
% % \begin{align}
% %     \sum_{p = 1}^P \log \frac{\det(A_{p, s}^{all})}{\det(A_{p-1, s}^{all})} = \log \det(A_{P,s}^{all}) \leq d \log(1+\frac{|\Psi_{T_c,s}|}{d}). \label{eqn:sum-log-epoch}
% % \end{align}
% % Inequality \eqref{eqn:sum-log-epoch} 
% indicates that the number of epochs with fewer than $\beta$ rounds is at most $O\left(\lceil \frac{R_s}{D/\beta} \rceil \right)$. Taking $\beta = \sqrt{\frac{DT_c}{R_s}}$ and note that $D = \frac{T_c \log(T_c)}{d^2 M}$, the total number of epochs for layer $s$ is at most $\lceil \frac{T_c}{\beta} \rceil + \lceil \frac{R_s \beta}{D} \rceil = O(\sqrt{\frac{T_cR_{s}}{D}}) = O(\sqrt{d^3 M})$, and the total communication cost is upper bounded by $O\left( S M \sqrt{d^3 M} \right) = O( \log(d) \sqrt{d^3 M^3} )$.

\begin{proof}[Proof of \cref{lem:syn_good}]
If $t$ is good and belongs to the $p$-th epoch, we have by \cref{lem:gramnorm} that
\begin{align}
\label{eqn:a-good-epoch}
    w_{t, s, a}^i = \alpha_{s} \Vert x_{t,a}^{i} \Vert_{(A_{t,s}^{i})^{-1} } \leq \sqrt{2} \alpha_{s} \Vert x_{t,a}^{i} \Vert_{(A_{p,s}^{all})^{-1} } \leq 2 \alpha_s \Vert x_{t,a}^{i} \Vert_{(A_{p-1,s}^{all})^{-1} }. 
\end{align}
Within $p$-th good epoch, we have
\begin{align*}
    A^{all}_{p-1, s} + \sum_{i=1}^M\sum_{t \in \text{$p$-th good epoch}} x^i_{t, a^i_t} (x^i_{t, a^i_t})^{\top} = A^{all}_{p, s}, 
\end{align*}
which together with inequality (\ref{eqn:a-good-epoch}) and the last inequality in the elliptical potential lemma (\cref{lem:epl}) imply that
\begin{align*}
    \sum_{i=1}^M\sum_{t \in \text{$p$-th good epoch}} \|x_{t,a_t^i}^{i}\|^2_{(A_{t,s}^{i})^{-1}} \leq 4\log \frac{\det(A^{all}_{p, s})}{\det(A^{all}_{p-1,s})}.
\end{align*}
Thus under event $\Ec$, the regret induced by good epochs of layer $s$ is 
\begin{align*}
    \sum_{(i,t) \in \Psi_{T, s}, t \text{ is good}} reg_{t, s}^i &\leq \sum_{(i,t) \in \Psi_{T, s}, t \text{ is good}} 8 w_{t, s, a_t^i}^i \\
    &\leq 8 \sqrt{ |\Psi_{T, s}| \sum_{(i,t) \in \Psi_{T, s}, t \text{ is good}} (w_{t, s, a_t^i}^i)^2 }\\
    &= \tilde{O}\left(\alpha_{s} \sqrt{ d |\Psi_{T,s}| \log(MT) } \right),
\end{align*}
where the first inequality is from \cref{lem:regret_lay}, the second inequality is by Cauchy-Schwartz inequality, and the last relation is from
\begin{align}
\label{equ:sum_good_chunk}
\sum_{p = 1}^P \log \frac{\det(A_{p, s}^{all})}{\det(A_{p-1, s}^{all})} = \log \det(A_{P,s}^{all}) \leq d \log(1+\frac{|\Psi_{T,s}|}{d}) = R_s. 
\end{align}
\end{proof}

\begin{proof}[Proof of \cref{lem:syn_bad}]
Denote by $R_s = d \log(1+\frac{|\Psi_{T,s}|}{d}) $. It follows that the number of bad epochs is at most $O(R_s)$. Moreover, the regret within a bad epoch of length $n$ can be upper bounded as $O(M +  \alpha_s M\sqrt{D})$ by applying the elliptical potential lemma for each client $i$ and the communication condition, where the extra $1$ in the upper bound is due to that at most $M$ clients trigger the communication condition at the end of the $p$-th epoch. We thus have
\begin{align*}
    \sum_{t \text{ is bad}} reg_{t, s}^i \leq \sum_{t \in \Psi_{T, s} \text{ is bad}} 8 w_{t, s}^i = O\left( M R_s + \alpha_s M \sqrt{D} R_s \right) = O\left(\alpha_s M \sqrt{D} R_s \right).
\end{align*}
\end{proof}

\begin{proof}[Proof of \cref{lem:Syn_Psi_bound}]
Recall $D = \frac{T \log(T)}{d^2 M}$. Note that if $ \alpha_s \sqrt{d |\Psi_{T, s}| \log(T)} = O(\alpha_s M \sqrt{D} R_s)$, we have $|\Psi_{T, s}| = \tilde{O}(M^2 D d) = \tilde{O}( \frac{M T}{d} )$. Otherwise $|\Psi_{T,s}| \overline{w}_{t}^{s} = O(\alpha_s \sqrt{d |\Psi_{T, s}| \log(T)})$, which implies $|\Psi_{T, s}| = \tilde{O}(\frac{\alpha_s^2 d}{(\overline{w}_{t}^{s})^2}) = \tilde{O}( \frac{MT 4^{s}}{d})$.
\end{proof}

\section{Variance-adaptive Async-\algg}
\label{appd:variance}

The variance-adaptive SupLinUCB subroutine is presented in Alg.~\ref{alg:VS-LUCB}, while the complete variance-adaptive Async-\alg is given in Alg.~\ref{alg:variance-AsynFedSupLinALgo}. 

\subsection{Algorithm}
\begin{algorithm}[!htb]
    \caption{Variance-adaptive SupLinUCB subroutine: VS-LUCB}
    \label{alg:VS-LUCB}
\begin{algorithmic}[1]
    \State \textbf{Initialization}: $S \leftarrow \lceil \log R +\log T \rceil $, $\overline{w}_{0} = d R^2$, $\overline{w}_{s} \leftarrow 2^{-s}\overline{w}_{0}, \forall s \in [1:S]$,\\
    $\alpha_0 = \Tilde{O}(\sqrt{d}), \alpha_s = 1 + \sqrt{2 \ln(2 K M T \ln d/\delta)}, \rho = 1/\sqrt{T}, \gamma = R^{1/2}/d^{1/4}$.
    \State \textbf{Input}: Client $i$ (with local information $A^i, b^i$, $\Delta A^i, \Delta b^i$), contexts set $\{x_{t,1}^i, \ldots, x_{t,K}^i\}$ 
    \State $A_{t, s}^i \leftarrow A_s^i, b_{t, s}^i \leftarrow b_s^i$ for \texttt{lazy update}
    \State $\hat{\theta}_{s} \leftarrow (A^i_{t, s})^{-1} b^{i}_{t, s}$, $\hat{r}_{t,s, a}^{i} = \hat{\theta}_s^\top x_{t, a}^i$,
    $ w_{t, s, a}^i \leftarrow \alpha_s \|x^i_{t,a}\|_{(A^i_{t, s})^{-1}}$, $\forall s \in [0 : S], \forall a \in [K]$.
    \State $s \leftarrow 0$; $\mathcal{A}_{0} \leftarrow \{ a\in [K] \mid \hat{r}_{t,0, a}^{i} + w_{t,0, a}^{i} \geq \max_{a \in [K]} (\hat{r}_{t,0, a}^{i} - w_{t,0, a}^{i}) \}$ \Comment{Initial screening}
   
    \Repeat  \Comment{Layered successive screening}
    \If{$s = S$}
    \State Choose action $a_t^i$ arbitrarily from $\mathcal{A}_{S}$
    \ElsIf{$w_{t,s, a}^{i} \leq \overline{w}_{s}$ for all $a \in \mathcal{A}_{s}$}
    \State $\mathcal{A}_{s+1} \leftarrow \{a \in \mathcal{A}_{s} \mid \hat{r}_{t,s,a}^{i} \geq    
	 \max _{a' \in \mathcal{A}_{s}} (\hat{r}_{t,s, a'}^{i} )-2 \overline{w}_{s} \}$; $s\leftarrow s+1$
    \Else
    \State Choose $a_t = \arg \max _{ \{a \in \mathcal{A}_{s},w_{t,s,a}^{i}>\overline{w}_s \}} w_{t,s,a}^{i}$
    \EndIf
    \Until{action $a_t$ is found}
    \State Take action $a_t$ and and {receive reward $r_{t,a_t}^{i}$ and variance $\sigma_t$}
    \State {$\overline{\sigma}_{t} = \max\{\sigma_t, \rho, 
    \gamma \Vert x_{t,a_t}^{i} \Vert_{ (A_{t,s}^{i})^{-1}}^{1/2} \}$}
    \State {$\Delta A_{s}^{i} \leftarrow \Delta A_{s}^{i} + x_{t,a_t}^{i}x_{t,a_t}^{i \top} / \overline{\sigma}_{t}^2$, $\Delta b_{s}^{i} \leftarrow \Delta b_{s}^{i} + r_{t,a_t}^{i} x_{t,a_t}^{i} / \overline{\sigma}_{t}^2 $ } \Comment{Update local information}
    \State Return layer index $s$
\end{algorithmic}
\end{algorithm}

\begin{algorithm}[!htb]
    \caption{Variance-adaptive Async-\algg}
    \label{alg:variance-AsynFedSupLinALgo}
\begin{algorithmic}[1]
    \State \textbf{Initialization}: $T$, $C$, $S = \lceil \log R +\log T  \rceil$
    \State $\{ A_{s}^{ser} \leftarrow I_{d}, b_{s}^{ser} \leftarrow 0 \mid s \in [0:S] \}$ \Comment{Server initialization}
    \State $\{ A_{s}^{i} \leftarrow I_{d},\Delta  A_{s}^{i} , b_{s}^{i}, \Delta b_{s}^{i} \leftarrow 0 \mid s \in [0:S], i\in [M]  \}$ \Comment{Clients initialization}
    \For{$t=1,2,...,T $}
    \State Client $i_t = i$ is active, and observes $K$ contexts $\{x_{t, 1}^{i}, x_{t, 2}^{i}, \cdots, x_{t, K}^{i}\}$
    \State $s =$ \texttt{VS-LUCB} $\left( client i, \{x_{t, 1}^{i}, x_{t, 2}^{i}, \cdots, x_{t, K}^{i}\} \right)$  with the lazy update
    \If{{$ \frac{\det(A_{s}^{i}+\Delta A_{s}^{i})}{ \det(A_{s}^{i}) } > (1+C) $}}
    \State $\sync$($s$, server, clients $i$) for each $s \in [0:S]$
    \EndIf 
    \EndFor
\end{algorithmic}
\end{algorithm}

\subsection{Supporting Lemmas and Proofs}
\begin{theorem}
\label{thm:variance-confidence}
(Theorem 4.3 in \citet{zhou2022computationally})
Let $\left\{\mathcal{F}_t\right\}_{t=1}^{\infty}$ be a filtration, and $\left\{x_t, \eta_t\right\}_{t \geq 1}$ be a stochastic process such that $x_t \in \mathbb{R}^d$ is $\mathcal{F}_t$-measurable and $\eta_t \in \mathbb{R}$ is $\mathcal{F}_{t+1}$-measurable. Let $ \sigma, \epsilon>0, \theta^* \in \mathbb{R}^d$. For $t \geq 1$, let $y_t=\left\langle \theta^*, x_t \right\rangle + \eta_t$ and suppose that $\eta_t, x_t$ also satisfy
$$ \mathbb{E} [ \eta_t \mid \mathcal{F}_t ]=0, \mathbb{E} [ \eta_t^2 \mid \mathcal{F}_t ] \leq \sigma^2,\left|\eta_t\right| \leq R,\left\|x_t\right\|_2 \leq 1.$$
For $t \geq 1$, let $Z_t= I+\sum_{i=1}^t x_i x_i^{\top}, b_t=\sum_{i=1}^t y_i x_i, \theta_t = Z_t^{-1} b_t$, and

\begin{align*}
\beta_t= & 12 \sqrt{\sigma^2 d \log \left(1+t L^2 /(d )\right) \log \left(32(\log (R / \epsilon)+1) t^2 / \delta\right)} \\
& +24 \log \left(32(\log (R / \epsilon)+1) t^2 / \delta\right) \max _{1 \leq i \leq t}\left\{\left|\eta_i\right| \min \left\{1,\left\|\mathbf{x}_i\right\|_{\mathbf{Z}_{i-1}^{-1}}^{-1}\right\}\right\}+6 \log \left(32(\log (R / \epsilon)+1) t^2 / \delta\right) \epsilon .
\end{align*}

Then, for any $0<\delta<1$, we have with probability at least $1-\delta$ that,
$$ \forall t \geq 1,\left\|\sum_{i=1}^t x_i \eta_i \right\|_{Z_t^{-1}} \leq \beta_t,\quad \| \theta_t-\theta^*\|_{Z_t} \leq \beta_t + \|\theta^* \|_2. $$
\end{theorem}

\begin{lemma}
\label{lem:zhoulemB1}
(Adapted from Lemma B.1 in \citet{zhou2022computationally}). Let $\left\{\sigma_t, \beta_t\right\}_{t \geq 1}$ be a sequence of non-negative numbers, $\rho, \gamma>0, \{ x_t \}_{t \geq 1} \subset \mathbb{R}^d$ and $\|x_t\|_2 \leq 1$. Let $\{Z_t \}_{t \geq 1}$ and $\left\{\bar{\sigma}_t\right\}_{t \geq 1}$ be recursively defined as follows: 
\begin{align*}
    Z_1 =  I;\quad  Z_{t+1} =Z_t + x_t x_t^{\top} / \bar{\sigma}_t^2, \quad \forall t \geq 1, \bar{\sigma}_t=\max \{\sigma_t, \rho, \gamma \|x_t \|_{Z_t^{-1}}^{1/2} \}.
\end{align*}
% $Z_1=  I$
% $$\forall t \geq 1, \bar{\sigma}_t=\max \{\sigma_t, \rho, \gamma \|x_t \|_{Z_t^{-1}}^{1/2} \},
% \quad Z_{t+1}=Z_t + x_t x_t^{\top} / \bar{\sigma}_t^2. $$
Let $\iota=\log (1+T/ (d\rho^2 ) )$. 
Then we have
$$ \sum_{t=1}^T \min \{1, \beta_t \| x_t \|_{Z_t^{-1}}\} \leq 2 d \iota+2 \beta_T \gamma^2 d \iota+2 \sqrt{d \iota} \sqrt{\sum_{t=1}^T \beta_t^2 (\sigma_t^2+\rho^2 )}.$$
\end{lemma}
Following a similar proof structure to Async-\algg, we employ a novel Bernstein-type self-normalized martingale inequality, proposed by \citet{zhou2022computationally}, for layer $0$ to manage the variance information. We define $\alpha_0 = \beta_T$ as specified in \cref{thm:variance-confidence}, and establish the following lemma, analogous to \cref{lem:event_layer0}.

\begin{lemma}
For any round $t \in [T]$, if client $i_{t} = i $ is active in round $t$ and arm $a_t$ is chosen in layer $0$, with probability at least $1-\delta $, with $\alpha_0 = \Tilde{O}(\sqrt{d})$ we have for any $a_t \in [K]$: 
$$ \left|\hat{r}_{t,0, a_t} -\theta^{\top} x_{t, a_t}^{i} \right| \leq w_{t,0,a_t}^{i} = \alpha_0
\| x_{t,a_t}^{i} \|_{(A_{t,0}^{i})^{-1}}. $$
\end{lemma}

We define \emph{good event} $\Ec$ as $\Ec  \triangleq \left\{ \left|x_{t,a}^{i \top} \hat{\theta}_{t,s}^i - x_{t,a}^{i \top} \theta\right| \leq w_{t,s,a}^{i}, \forall i \in [M], a \in[K], t \in[T], s\in [0:S]  \right\}.$
In a manner similar to the proof of \cref{lem:goodevent}, we have that $\Pb [\Ec] \geq 1 - \delta$.

\begin{lemma}
\label{lem:variance-regret-layer0}
Conditioned on the event $\Ec$, the regret in layer $0$ can be bounded by $\ssf{reg}_{\text{layer }0}\leq \tilde{O}(d).$
\end{lemma}

\begin{proof}[Proof of \cref{lem:variance-regret-layer0}]
We set $\overline{w}_{0} = dR^2$ to provide a tighter bound for the size of $\Psi_{T,0}$. Mirroring the proof methodology in \cref{lem:Psi_0_bound}, we establish the following:
\begin{align*}
\overline{w}_0 |\Psi_{T,0}| 
&\leq \alpha_0 \sum_{t\in \Psi_{T,0}} \Vert x_{t,i}\Vert_{(A_{t,s}^{i_t})^{-1}}
\leq  2 d \iota+2 \alpha_0 \gamma^2 d \iota+2 \alpha_0 \sqrt{d \iota} \sqrt{\sum_{t \in \Psi_{T,0}} (\sigma_t^2+\rho^2 )} \\
& \leq  2 d \iota+2 \alpha_0 \gamma^2 d \iota+2 \alpha_0 \sqrt{d \iota}  \sqrt{|\Psi_{T,0}| (R^2+\rho^2 )}.
\end{align*}
The first inequality results from the arm selection rule of layer 0, the second is derived from \cref{lem:zhoulemB1}, and the third arises due to the constraint $\sigma_t^2 \leq R^2$. Consequently, we infer that $\Psi_{T,0} \leq O( d^2 R^2 / \overline{w}_0^2)$. We can then bound the regret in layer $0$ as follows:
\begin{align*}
\ssf{reg}_{\text{layer } 0}
&\leq  4  \alpha_0 \sum_{t\in \Psi_{T,0}} \Vert x_{t,i}\Vert_{(A_{t,s}^{i_t})^{-1}} 
\leq  8 d \iota+ 8 \alpha_0 \gamma^2 d \iota+ 8 \alpha_0 \sqrt{d \iota} \sqrt{|\Psi_{T,0}| (R^2+\rho^2 )} \leq \Tilde{O} (d).
\end{align*}
\end{proof}

\begin{lemma}
\label{lem:variance-regret-layers}
Conditioned on the event $\Ec$, the regret of each layer  $s\in [1:S-1]$ can be bounded by $\ssf{reg}_{\text{layer } s} \leq \tilde{O} \sqrt{ d\sum_{t} \sigma_t^2 }.$
\end{lemma}

\begin{proof}[Proof of \cref{lem:variance-regret-layers}]
For $s \in \{1,2,...,S-1\}$, the rewards in each layer $s$ are mutually independent, as proven in \cref{lem:SupLin-independence}. We deduce:
\begin{align*}
\ssf{reg}_{\text{layer } s}
& \leq 8 \overline{w}_{s} |\Psi_{T,s}| \leq 8 \alpha_s \sum_{t\in \Psi_{T,s}} \Vert x_{t,i}\Vert_{(A_{t,s}^{i_t})^{-1}} \\ 
& \leq  \alpha_s \sum_{t\in \Psi_{T,s}} \| x_{t,a}^{i_{t}} \|_{(A_{t,s}^{all})^{-1}} +\alpha_s d M\log(1+T/d)\\
& \leq \Tilde{O}(\sqrt{d \sum_{t \in \Psi_{T,s}} \sigma_t^2}).
\end{align*}
The first inequality arises from \cref{{lem:regret_lay}}, the second is a result of the arm selection rule in Line 13, the third derives from \cref{lem:width_lay}, and the final inequality is attributable to \cref{lem:zhoulemB1}.
\end{proof}

For the final layer $S$, applying \cref{lem:regret_lay} and setting $\overline{w}_{S} = d/T$, we have $\ssf{reg}_{\text{layer } S}  \leq 8\overline{w}_s |\Psi_{S}| \leq  \Tilde{O} (d).$

\paragraph{Proof of the communication bound in \cref{thm:async-variance}.} 
Having established the bound for regret in each layer, we have demonstrated that $R_T \leq \tilde{O}\left(\sqrt{d \sum_{t = 1}^T \sigma_t^{2} }\right)$. Given that we set $\overline{w}_{0} = d R^2$ and $\overline{w}_{S} = d/T$, it requires $S = \log( \overline{w}_0/\overline{w}_S ) = \Theta (\log R +\log T )$ layers to achieve the desired accuracy.  The number of communications triggered by layer $s$ can be upper bounded by $O(dM^2 \log(T)$  (\cref{lem:comm_layer}). Consequently, we are able to constrain the overall communication cost to $\Tilde{O}( d M^2 \log^2 T)$.

\section{Corruption Robust Async-\algg}
\label{appd:corruption}

The corruption robust SupLinUCB subroutine is presented in Alg.~\ref{alg:CS-LUCB}, while the complete corruption robust Async-\alg is given in Alg.~\ref{alg:Corruption-AsynFedSupLinALgo}. 
  
\subsection{Algorithm}
\begin{algorithm}[ht]
    \caption{Corruption Robust SupLinUCB subroutine: CS-LUCB}
    \label{alg:CS-LUCB}
\begin{algorithmic}[1]
    \State \textbf{Initialization}: $S = \lceil \log d \rceil $, 
    $\overline{w}_{0} = d^{1.5}/\sqrt{T}$, $\overline{w}_{s} \leftarrow 2^{-s}\overline{w}_{0}$, $\gamma = \sqrt{d}/C_p$.\\
    $\alpha_{0} = 1 + \sqrt{d \ln(2M^2 T/\delta)} + \gamma C_p, \alpha_{s} \leftarrow 1 + \sqrt{2 \ln(2 K M T \ln d/\delta)} +\gamma C_p, \forall s \in [1:S]$
    \State \textbf{Input}: Client $i$ (with local information $A^i, b^i$, $\Delta A^i, \Delta b^i$), contexts set $\{x_{t,1}^i, \ldots, x_{t,K}^i\}$
    \State $A_{t, s}^i \leftarrow A_s^i, b_{t, s}^i \leftarrow b_s^i$ for \texttt{lazy update}
    \State $\hat{\theta}_{s} \leftarrow (A^i_{t, s})^{-1} b^{i}_{t, s}$, $\hat{r}_{t,s, a}^{i} = \hat{\theta}_s^\top x_{t, a}^i$,
    $ w_{t, s, a}^i \leftarrow \alpha_s \|x^i_{t,a}\|_{(A^i_{t, s})^{-1}}$, $\forall s \in [0 : S], \forall a \in [K]$.
    \State $s \leftarrow 0$; $\mathcal{A}_{0} \leftarrow \{ a\in [K] \mid \hat{r}_{t,0, a}^{i} + w_{t,0, a}^{i} \geq \max_{a\in [K]} (\hat{r}_{t,0, a}^{i} - w_{t,0, a}^{i}) \}$. \Comment{Initial screening}
    \Repeat \Comment{Layered successive screening}
    \If{$s = S$}
    \State	Choose action $a_t^i$ arbitrarily from $\mathcal{A}_{S}$
    \ElsIf{$w_{t,s, a}^{i} \leq \overline{w}_{s}$ for all $a \in \mathcal{A}_{s}$}
    \State %Reduce feasible set 
	$\mathcal{A}_{s+1} \leftarrow \{a \in \mathcal{A}_{s} \mid \hat{r}_{t,s,a}^{i} \geq    
	 \max_{a' \in \mathcal{A}_{s}} (\hat{r}_{t,s, a'}^{i} )-2 \overline{w}_{s} \}$; $s\leftarrow s+1$
    \Else
    \State $a_t^i \leftarrow \arg \max _{ \{a \in \mathcal{A}_{s},w_{t,s,a}^{i}>\overline{w}_s \}} w_{t,s,a}^{i}$
    \EndIf 
    \Until{action $a_t^i$ is found}
    \State Take action $a_t^i$ and and receive reward $r_{t,a_t^i}^{i}$
    \State $\eta_t = \min \{1, \gamma/ \|x^i_{t,a_t}\|_{(A^i_{t, s})^{-1}} \}$
    \State $\Delta A_{s}^{i} \leftarrow \Delta A_{s}^{i} + \eta_t x_{t,a_t^i}^{i}x_{t,a_t^i}^{i \top}$, $\Delta b_{s}^{i} \leftarrow \Delta b_{s}^{i} + \eta_t r_{t,a_t^i}^{i} x_{t,a_t^i}^{i}$ \Comment{Update local information}
    \State \textbf{Return} layer index $s$
\end{algorithmic}
\end{algorithm}

\begin{algorithm}[!htb]
    \caption{Corruption Robust Async-\algg}
    \label{alg:Corruption-AsynFedSupLinALgo}
\begin{algorithmic}[1]
    \State \textbf{Initialization}: $T$, $C$, $S = \lceil \log d \rceil$
    \State $\{ A_{s}^{ser} \leftarrow I_{d}, b_{s}^{ser} \leftarrow 0 \mid s \in [0:S] \}$ \Comment{Server initialization}
    \State $\{ A_{s}^{i} \leftarrow I_{d},\Delta  A_{s}^{i} , b_{s}^{i}, \Delta b_{s}^{i} \leftarrow 0 \mid s \in [0:S], i\in [M]  \}$ \Comment{Clients initialization}
    \For{$t=1,2,\cdots,T $}
    \State Client $i_t = i$ is active, and observes $K$ contexts $\{x_{t, 1}^{i}, x_{t, 2}^{i}, \cdots, x_{t, K}^{i}\}$
    \State $s \leftarrow$ \texttt{CS-LUCB} $\left( \text{client }i, \{x_{t, 1}^{i}, x_{t, 2}^{i}, \cdots, x_{t, K}^{i}\} \right)$  with lazy update
    \If{{$ \frac{\det(A_{s}^{i}+\Delta A_{s}^{i})}{ \det(A_{s}^{i}) } > (1+C) $}}
    \State $\sync$($s$, server, clients $i$) for each $s \in [0:S]$
    \EndIf 
    \EndFor
\end{algorithmic}
\end{algorithm}

\subsection{Supporting Lemmas and Proof}
When confronted with adversarial corruption, we utilize a weighted ridge regression in which the weight assigned to each selected action depends on its confidence. Further, we expand the confidence width to accommodate this corruption, with $\alpha_{0} = 1 + \sqrt{d \ln(2M^2 T/\delta)} + \gamma C_p$ and $\alpha_{s} = 1 + \sqrt{2 \ln(2 K M T \ln d/\delta)} +\gamma C_p$ as proposed in \citet{he2022nearly}. In our analysis of layer $0$, we adapt Lemma B.1 from \citet{he2022nearly} to fit a federated scenario, yielding the following lemma:

\begin{lemma}
(Adapted from Lemma B.1  in \citet{he2022nearly})
\label{lem:corruption_he_layer0}
% (Local confidence bound) 
Under the setting of \cref{thm:async}, in the layer $0$, with probability at least $1-\delta$, the following event $\Ec_{0}$ happens:
$$
\Ec_{0}  \triangleq \{ \left|x_{t,a}^{i \top} \hat{\theta}_{t,s}^i - x_{t,a}^{i \top} \theta\right| \leq w_{t,s,a}^{i}, \forall i \in [M], a \in[K], t \in[T], s = 0 \}.
$$
\end{lemma}

For each layer $s \in [S]$, the rewards are mutually independent, analogous to the proof of \cref{lem:SupLin-independence}. We can restate the lemma as follows:

\begin{lemma}
\label{lem:corruption_con_layers}
Suppose the time index set $\Psi_{t,s}$ is constructed so that for fixed $x_{\tau,a_{\tau}}$ with $\tau \in \Psi_{t,s}$, the rewards $\{r_{\tau,a_{\tau}}\}$ are independent random variables with means $\mathbb{E}[r_{\tau,a_{\tau}}] = \theta^{\top} x_{\tau,a_{\tau}} + c_{\tau} $. For any round $t \in [T]$, if client $i_{t} = i $ is active and chooses arm $a_t$ in layer $s \in [S]$, with probability at least $1-\frac{\delta}{MT\ln d}$, we have for any $a_t \in [K]$: 
$$ \left|\hat{r}_{t,s,a_t} -\theta^{\top} x_{t, a_t}^{i }  \right| \leq w_{t,s,a_t}^{i} = \alpha_s 
\| x_{t,a_t}^{i} \|_{(A_{t,s}^{i})^{-1}} .$$
\end{lemma}

After combining the aforementioned events, we redefine the good event in the presence of corruption as follows:
$$\Ec  \triangleq \left\{ \left|x_{t,a}^{i \top} \hat{\theta}_{t,s}^i - x_{t,a}^{i \top} \theta\right| \leq w_{t,s,a}^{i}, \forall i \in [M], a \in[K], t \in[T], s\in [0:S]  \right\}.$$ 
Similar to proof of \cref{lem:goodevent}, we have that $\Pb [\Ec] \geq 1 - \delta$. 

\begin{lemma}
\label{lem:corruption-regret-layer0}
Conditioned on the good event $\Ec$, the regret of layer $s\in [0:S-1]$ can be bounded as follows: $\ssf{reg}{\text{layer} s} \leq \tilde{O}(\sqrt{ d T }+ dC_p)$.
\end{lemma}

\begin{proof}[Proof of \cref{lem:corruption-regret-layer0}]
Under the condition of the good event $\Ec$, we adopt a similar approach to the regret decomposition analysis presented in \citet{he2022nearly} to bound the regret in each layer $s \in [0:S-1]$.
\begin{align}
& \mathbb{E} \sum_{t \in \Psi_{T, s}} (r^i_{t,a_t^{i, *}}- r^i_{t,a_t}) \leq \sum_{t \in \Psi_{T, s}} 8 w_{t,s,a_t}
= \sum_{t \in \Psi_{T, s}} 8 \alpha_s  \|x^i_{t,a}\|_{(A^i_{t, s})^{-1}} \\
& = \underbrace{8 \alpha_s \sum_{t \in \Psi_{T, s}, \eta_t = 1}   \|x^i_{t,a}\|_{(A^i_{t, s})^{-1}}}_{I_1}+ \underbrace{ 8 \alpha_s  \sum_{t \in \Psi_{T, s}, \eta_t < 1}  \|x^i_{t,a}\|_{(A^i_{t, s})^{-1}}}_{I_2}. \label{eqn:I1I2}
\end{align}

The first inequality is derived from \cref{lem:regret_lay}, while \cref{eqn:I1I2} follows from the definition of $\eta_t$. For the term $I_1$, we consider the rounds with $\eta_t = 1$, assuming these rounds can be listed as $\{ k_1, k_2, ..., k_n \}$.  To analyze this, we construct the auxiliary matrix $B_{t,s} = I + \sum_{j=1}^{n} x_{k_j} x_{k_j}^{\top} I\{k_j \leq t \}$. Using the definition of $A_{t,s}^{i}$, we can establish the inequality $A_{t,s}^{i} \succeq \frac{1}{1+M C} A_{t,s}^{all} \succeq \frac{1}{1+M C} B_{t,s}$. 

Then we have
\begin{align*}
& I_1 = \sum_{t \in \Psi_{T, s}, \eta_t = 1} 8 \alpha_s  \|x^i_{t,a}\|_{(A^i_{t, s})^{-1}} \\
& \leq 8 \alpha_s \sqrt{2(1+MC)} \sqrt{2 d  |\Psi_{T,s}| \log |\Psi_{T,s}|} + 8 \alpha_s d M\log(1+T/d)
\leq \tilde{O}( \sqrt{dT}),
\end{align*}
where the first inequality follows from \cref{lem:width_lay}, and the second inequality is obtained by noting that the size of $\Psi_{T,0}$ is bounded by $\Tilde{O}(T/d)$, as stated in \cref{lem:Psi_0_bound} particularly for layer $0$. 

For the term $I_2$, using the property $\eta_t< 1$, we can express $\eta_t$ as $\eta_t = \gamma/ \|x^i_{t,a}\|_{(A^i_{t, s})^{-1}}$, which implies:
\begin{align*}
& I_2 = \sum_{t \in \Psi_{T, s}, \eta_t < 1} 8 \alpha_s  \|x^i_{t,a}\|_{(A^i_{t, s})^{-1}} \\
& \leq  \sum_{t \in \Psi_{T, s}, \eta_t < 1} 8 \frac{\alpha_s}{\gamma}  \eta_t x^{i \top}_{t,a}  (A^i_{t, s})^{-1} x^i_{t,a} \leq \frac{\alpha_s}{\gamma}d\log(T) \leq \Tilde{O}(d C_p),
\end{align*}
where the first inequality is derived from the definition of $\eta_t$, the second inequality is obtained from the elliptical potential lemma, as referenced in \cref{lem:epl}, and the third inequality stems from the definition of $\alpha_s$. 

By combining $I_1$ and $I_2$, we can ultimately bound the regret in each layer $s\in [0:S-1]$ as $\ssf{reg}_{\text{layer} s} \leq \tilde{O}(\sqrt{ d T }+ dC_p)$.
\end{proof}

For the regret that occurs in the last layer $S$, we can derive the following bound:

\begin{align*}
\sum_{t \in \Psi_{T,S}} \mathbb{E} \left[ r^i_{t,a_t^{i, *}}- r^i_{t,a_t} \right] \leq \sum_{t \in \Psi_{T,S}} 8 \overline{w}_{S}  \leq 8 \overline{w}_{S} |\Psi_{T,S}|\leq 8\overline{w}_{S} T  \leq 8\sqrt{dT}.
\end{align*}
The first inequality is from \cref{lem:regret_lay}, and the last inequality follows from $\overline{w}_S = \sqrt{d/T}$. 

\paragraph{Proof of the communication bound in \cref{thm:async-corruption}.} 
By combining the regret in each layer, we can conclude that $R_{T} \leq \tilde{O}(\sqrt{ d T }+ dC_p)$. Note that, based on the definition of $\eta_t \leq 1$ and \cref{lem:epl}, it follows that $\log (\det(A_{t,s}^{all}) ) \leq d \log(1+|\Psi_{T,s}|/d)$. Additionally, by following a similar proof as in \cref{lem:comm_layer}, we can bound the number of communication rounds in layer $s$ by $O(d M^2 \log T)$. Considering that the FedSupLinUCB algorithm has $S=\lceil \log d \rceil$ layers, the total communication cost is therefore upper bounded by $O(d M^2 \log d \log T)$.

\end{document}